\documentclass[letterpaper]{article} 
\usepackage[preprint]{aaai2027}  
\usepackage[hyphens]{url}  
\usepackage{graphicx} 
\usepackage{natbib}  
\usepackage{caption} 
\usepackage{amsmath, amssymb}
\usepackage{amsthm}
\newtheorem{theorem}{Theorem}
\usepackage{booktabs}
\usepackage{multirow}
\usepackage{placeins}

\usepackage{algorithm}
\usepackage{algorithmic}
\usepackage{newfloat}
\usepackage{listings}
\usepackage{nameref}  

\newcommand{\sampleGraphic}[1]{%
  \IfFileExists{#1.jpg}{%
    \includegraphics[width=\linewidth]{#1.jpg}%
  }{%
    \IfFileExists{compressed_samples/#1.jpg}{%
      \includegraphics[width=\linewidth]{compressed_samples/#1.jpg}%
    }{%
      \includegraphics[width=\linewidth]{#1.png}%
    }%
  }%
}
\newcommand{\qwenSampleRow}[2]{%
  \noindent\begin{minipage}[t]{0.32\linewidth}\centering
    \sampleGraphic{qwenbase_hps_clip/checkpoint-71-0/#1}\\
    {\scriptsize Baseline}
  \end{minipage}\hfill
  \begin{minipage}[t]{0.32\linewidth}\centering
    \sampleGraphic{qwenjagg_hps_clip/checkpoint-71-0/#1}\\
    {\scriptsize JAGG}
  \end{minipage}\hfill
  \begin{minipage}[t]{0.32\linewidth}\centering
    \sampleGraphic{qwenjagg_ablate_firstlast_hps_clip/checkpoint-71-0/#1}\\
    {\scriptsize JAGG (ablate)}
  \end{minipage}\\[2pt]
  {\footnotesize\textbf{\#\,#1:} \emph{#2}}\par\vspace{8pt}
}
\newcommand{\fluxSampleRow}[2]{%
  \noindent\begin{minipage}[t]{0.48\linewidth}\centering
    \sampleGraphic{flux_hps_clip_base/checkpoint-81-0/#1}\\
    {\scriptsize Baseline}
  \end{minipage}\hfill
  \begin{minipage}[t]{0.48\linewidth}\centering
    \sampleGraphic{flux_hps_clip_jagg/checkpoint-81-0/#1}\\
    {\scriptsize JAGG}
  \end{minipage}\\[2pt]
  {\footnotesize\textbf{\#\,#1:} \emph{#2}}\par\vspace{8pt}
}
\newcommand{\flowFastQwenSampleRow}[3]{%
  \noindent\begin{minipage}[t]{0.48\linewidth}\centering
    \sampleGraphic{figures/flowgrpo_fast_qwenimage_jagg_vs_baseline_step181/#1/baseline_#2}\\
    {\scriptsize Baseline (181 steps)}
  \end{minipage}\hfill
  \begin{minipage}[t]{0.48\linewidth}\centering
    \sampleGraphic{figures/flowgrpo_fast_qwenimage_jagg_vs_baseline_step181/#1/jagg_#2}\\
    {\scriptsize JAGG (181 steps)}
  \end{minipage}\\[2pt]
  {\footnotesize\textbf{\#\,#2:} \emph{#3}}\par\vspace{8pt}
}
\newcommand{\flowFastFluxSampleRow}[3]{%
  \noindent\begin{minipage}[t]{0.32\linewidth}\centering
    \sampleGraphic{figures/flowgrpo_fast_flux_jagg_ablate_vs_baseline_step181/#1/baseline_#2}\\
    {\scriptsize Baseline}
  \end{minipage}\hfill
  \begin{minipage}[t]{0.32\linewidth}\centering
    \sampleGraphic{figures/flowgrpo_fast_flux_jagg_ablate_vs_baseline_step181/#1/jagg_#2}\\
    {\scriptsize JAGG}
  \end{minipage}\hfill
  \begin{minipage}[t]{0.32\linewidth}\centering
    \sampleGraphic{figures/flowgrpo_fast_flux_jagg_ablate_vs_baseline_step181/#1/ablation_#2}\\
    {\scriptsize JAGG (ablate)}
  \end{minipage}\\[2pt]
  {\footnotesize\textbf{\#\,#2:} \emph{#3}}\par\vspace{8pt}
}

\DeclareCaptionStyle{ruled}{labelfont=normalfont,labelsep=colon,strut=off} 
\floatstyle{ruled}
\newfloat{listing}{tb}{lst}{}
\floatname{listing}{Listing}
\title{JAGG: Jacobian-Aggregated Group Gradient for Efficient GRPO Training of Diffusion Models}
\author{
    Ruiyi Ding\textsuperscript{\rm 1}\thanks{Work done during internship at Kling AI Infra},
    Jie Li\textsuperscript{\rm 1},
    Kang He\textsuperscript{\rm 1},
    Ziyan Liu\textsuperscript{\rm 5},
    Chengru Song\textsuperscript{\rm 1}\corresponding,
    Yuan Chen\textsuperscript{\rm 2,3,4}\corresponding
}
\affiliations{
    \textsuperscript{\rm 1}KlingAI Research\\
    \textsuperscript{\rm 2}Fudan University\quad
    \textsuperscript{\rm 3}Shanghai AI Incubation and Innovation Center\\
    \textsuperscript{\rm 4}Shanghai Academy of AI for Science\quad
    \textsuperscript{\rm 5}University of Science and Technology of China\\
    dingruiyi@klingai.com, songchengru@klingai.com, cheng\_yuan@fudan.edu.cn
}

\begin{document}

\maketitle

\begin{abstract}

Group Relative Policy Optimization (GRPO) is a powerful reinforcement learning algorithm for aligning generative models with human preferences. While successful in large language models~\cite{shao2024deepseekmathpushinglimitsmathematical}, its extension to diffusion and flow matching models introduces a severe computational bottleneck: gradients must be back-propagated through the high-capacity DiT backbone at \emph{every} timestep of the sampling trajectory, making high-resolution text-to-image (T2I) training prohibitively expensive. Training-free DiT inference acceleration methods (e.g., $\Delta$-DiT, ScalingCache) exploit the fact that DiT hidden states and velocity predictions vary \emph{smoothly and nearly linearly} along the trajectory. We ask whether the same linearity can reduce the backward-pass cost of DiT RL training, and answer affirmatively with \textbf{JAGG} (\textbf{J}acobian-\textbf{A}ggregated \textbf{G}roup \textbf{G}radient), which reduces full transformer backward passes from $W$ to $2$ per group of $W$ consecutive steps. JAGG approximates intermediate-step Jacobians via $t$-weighted interpolation of the endpoint Jacobians, then aggregates per-step upstream signals into two composite gradients applied through a single joint backward pass. We prove this interpolation is \emph{exact} when the velocity is linear in $(z,t)$, and a cosine-similarity routing rule (\texttt{jagg\_frac}) deploys JAGG only where the assumption holds. Experiments on T2I benchmarks show JAGG delivers $\sim$2$\times$ backward speedup with negligible quality degradation.
\end{abstract}

\section{Introduction}

Group Relative Policy Optimization (GRPO) has recently emerged as a compelling reinforcement learning (RL) algorithm for aligning generative models with human feedback and reward signals. Initially proposed in the context of large language models (LLMs)~\cite{shao2024deepseekmathpushinglimitsmathematical}, GRPO eliminates the need for a separate value network by normalizing advantages \emph{within} a group of candidate responses sampled from the same policy. This design significantly reduces the memory and computational overhead compared to Proximal Policy Optimization (PPO), making online RL training practical for large-scale autoregressive models. The success of GRPO in mathematical reasoning~\cite{shao2024deepseekmathpushinglimitsmathematical}, long-context tasks~\cite{ping2026longrunleashinglongcontextreasoning}, and multimodal alignment~\cite{wang2026prismprealignmentblackboxonpolicy} has spurred interest in extending the paradigm beyond language.

More recently, GRPO has been adapted to diffusion-based visual generation, where it has shown promising results in improving image quality, text alignment, and aesthetic preference. Two pioneering works---DanceGRPO~\cite{xue2025dancegrpounleashinggrpovisual} and Flow-GRPO~\cite{liu2025flowgrpotrainingflowmatching}---demonstrated that applying GRPO to diffusion transformers (DiTs) and flow matching models can substantially enhance generation quality. Both methods treat the denoising process as a multi-step decision-making problem, computing a policy gradient loss at each timestep of the sampling trajectory. However, this per-step formulation introduces a severe computational bottleneck: standard training requires one full forward~+~backward pass through the denoising network for \emph{every} timestep in each training window. For high-resolution text-to-image (T2I) generation, where sampling trajectories typically span dozens of steps, this leads to prohibitively expensive training.

The key distinguishing feature of DiT RL, compared to language model RL, is that the ``actions'' are entire high-dimensional denoising steps, and the policy must be optimized over long multi-step trajectories that unfold through the DiT backbone. This structure imposes an unavoidable requirement: \textbf{gradients must be back-propagated through the denoising network at every timestep of the trajectory} in order to update model parameters. Unlike language models, where each token step is lightweight, each backward pass through a modern DiT incurs the full cost of a transformer forward pass, making \textbf{the per-step backward requirement the dominant training expense}.

Interestingly, this same trajectory structure has been extensively studied in the context of \emph{training-free} DiT inference acceleration. A line of work has observed that \textbf{the hidden states, attention outputs, and velocity predictions of DiTs vary \emph{smoothly and nearly linearly} along the denoising trajectory}: consecutive denoising steps produce similar intermediate representations that can be reused, cached, or extrapolated with minimal quality loss. Methods such as $\Delta$-DiT~\cite{chen2024deltadittrainingfreeaccelerationmethod}, Pyramid Attention Broadcast~\cite{zhao2025realtimevideogenerationpyramid}, TaylorSeers~\cite{liu2025reusingforecastingacceleratingdiffusion}, and ScalingCache~\cite{gu2026scalingcache} all exploit this trajectory smoothness to skip or approximate network evaluations at intermediate timesteps, achieving substantial inference speedups with negligible quality degradation.

This raises a natural question: \textbf{can the same trajectory linearity that benefits training-free inference acceleration also be exploited to reduce the backward-pass cost in DiT RL training?} If the Jacobians of the denoising velocity field with respect to model parameters vary linearly along the trajectory---mirroring the observed smoothness of the forward activations---then the expensive per-step backward passes could be replaced by cheap interpolations, reducing the total number of full network backward passes from $W$ to a small constant.

In this work, we address the efficiency challenge of diffusion GRPO training. We propose \textbf{JAGG} (\textbf{J}acobian-\textbf{A}ggregated \textbf{G}roup \textbf{G}radient), a mathematically principled acceleration method that reduces the number of network backward passes from $W$ to $2$ for a group of $W$ consecutive denoising steps. Our key insight is that the Jacobian of the denoising velocity field $v_\theta(z,t)$ with respect to model parameters $\theta$ is approximately linear in the inputs $(z,t)$ during the high-noise regime, which permits accurate interpolation of intermediate-step Jacobians from only the endpoint Jacobians. We provide a formal proof that this interpolation is \emph{exact} when the velocity is strictly linear, and demonstrate empirically that the approximation holds with negligible error under practical training conditions.

We validate JAGG on T2I benchmarks, demonstrating consistent acceleration with negligible degradation in generation quality. Our contributions are:
\begin{itemize}
    \item We identify and formalize the computational bottleneck in diffusion GRPO training, arising from per-timestep backward passes across the sampling trajectory.
    \item We propose JAGG, a Jacobian-aggregation framework that reduces backward passes from $W$ to $2$ per group, grounded in a rigorous theoretical analysis of the linearity of velocity Jacobians.
    \item We prove that under linear velocity conditions, the $t$-weighted Jacobian interpolation is exact with zero approximation error.
    \item We demonstrate the generality of our method through extensive experiments on text-to-image generation tasks across multiple diffusion backbones.
\end{itemize}

\section{Related Work}

\noindent\textbf{Diffusion, Flow Matching, and RL Fine-tuning.} Diffusion models~\cite{ho2020denoisingdiffusionprobabilisticmodels,song2022denoisingdiffusionimplicitmodels} and rectified-flow matching~\cite{liu2022rectifiedflowstraightfastlearning} underpin modern visual generation, with surveys of their interplay with RL in~\cite{zhu2024diffusionmodelsreinforcementlearning,10529221}. Early RL-based fine-tuning frames denoising as a multi-step decision process (DDPO~\cite{black2024trainingdiffusionmodelsreinforcement}), and follow-ups explore preference optimization (DPOK~\cite{NEURIPS2023_fc65fab8}, DiffPO~\cite{chen2025diffpodiffusionstyledpreferenceoptimization}), KL-regularized PPO (DiffPPO~\cite{10987802}), text-encoder RL (TexForce~\cite{10.1007/978-3-031-72698-9_11}), and diversity-oriented RL~\cite{10656815}.

\noindent\textbf{GRPO for Visual Generation.} DanceGRPO~\cite{xue2025dancegrpounleashinggrpovisual} and Flow-GRPO~\cite{liu2025flowgrpotrainingflowmatching} first adapted GRPO to diffusion/flow models, with follow-ups on unified reasoning-driven optimization (UniGRPO~\cite{liu2026unigrpounifiedpolicyoptimization}), stability via branching (BranchGRPO~\cite{li2025branchgrpostableefficientgrpo}) or clip regulation (GRPO-Guard~\cite{wang2025grpoguardmitigatingimplicitoveroptimization}), ODE/SDE mixing (MixGRPO~\cite{li2026mixgrpounlockingflowbasedgrpo}), credit assignment (PCPO~\cite{lee2026pcpoproportionatecreditpolicy}), and extensions to masked~\cite{luo2026reinforcementlearningmeetsmasked}, autoregressive~\cite{yuan2025argrpotrainingautoregressiveimage}, and discrete~\cite{wang2026udmgrpostableefficientgroup,ma2025consolidatingreinforcementlearningmultimodal} generators, video RL~\cite{zheng2026manifoldawareexplorationreinforcementlearning}, simplified online RL for denoisers (V-GRPO~\cite{tang2026vgrpoonlinereinforcementlearning}), forward-process alignment (DiffusionNFT~\cite{zheng2026diffusionnftonlinediffusionreinforcement}), likelihood-estimation design~\cite{choi2026rethinkingdesignspacereinforcement}, and black-box distillation (PRISM~\cite{wang2026prismprealignmentblackboxonpolicy}). Beyond images, GRPO has also been applied to diffusion language models~\cite{zhong2026stabilizingreinforcementlearningdiffusion,rojas2026improvingreasoningdiffusionlanguage,zhao2025inpaintingguidedpolicyoptimizationdiffusion}. Existing accelerations of the GRPO backward pass---stochastic backward in DanceGRPO~\cite{xue2025dancegrpounleashinggrpovisual}, restricted SDE windows in Flow-GRPO~\cite{liu2025flowgrpotrainingflowmatching} and MixGRPO~\cite{li2026mixgrpounlockingflowbasedgrpo}, and merged trajectories in BranchGRPO~\cite{li2025branchgrpostableefficientgrpo}---are heuristic and discard much per-step gradient information; we instead \textbf{faithfully approximate the full backward-pass signal} at a substantially lower cost.

\noindent\textbf{GRPO in Language Models.} GRPO originated in DeepSeekMath~\cite{shao2024deepseekmathpushinglimitsmathematical} and was later extended to long-context reasoning~\cite{ping2026longrunleashinglongcontextreasoning} and connected to classification-style RLVR~\cite{zhai2026rewardslabelsrevisitingrlvr}; see~\cite{zhang2025surveyreinforcementlearninglarge} for a broader survey. Unlike inference-time acceleration (distillation, consistency models), our work is orthogonal and targets the backward-pass bottleneck specific to diffusion GRPO.

\section{Method}

In this section, we first present the preliminaries of diffusion and flow models, then introduce the GRPO policy loss and advantage formulation, and finally detail our JAGG acceleration method with its theoretical justification.

\subsection{Preliminaries}

\textbf{Diffusion and Flow Models.} Diffusion models~\cite{ho2020denoisingdiffusionprobabilisticmodels} define $\mathbf{z}_t=\alpha_t\mathbf{x}+\sigma_t\boldsymbol{\epsilon}$ with noise schedule $(\alpha_t, \sigma_t)$. Rectified flows~\cite{liu2022rectifiedflowstraightfastlearning} use linear interpolation $\mathbf{z}_t=(1-t)\mathbf{x}+t\boldsymbol{\epsilon}$ with velocity $\mathbf{u}=\boldsymbol{\epsilon}-\mathbf{x}$. Both reverse via denoising steps: for diffusion, $\mathbf{z}_s=\alpha_s\hat{\mathbf{x}}+\sigma_s\hat{\boldsymbol{\epsilon}}$ (DDIM~\cite{song2022denoisingdiffusionimplicitmodels}); for flow, $\mathbf{z}_s=\mathbf{z}_t+\hat{\mathbf{u}}\cdot(s-t)$. The generative process can be stochastic (SDE) or deterministic (ODE). Following~\cite{albergo2023buildingnormalizingflowsstochastic,albergo2025stochasticinterpolantsunifyingframework}, the generative SDEs are:
\begin{equation}
\label{eq:sde_rf}
    \mathrm{d}\mathbf{z}_t=(\mathbf{u}_t-\tfrac{1}{2}\varepsilon_t^2\nabla\log p_t(\mathbf{z}_t))\mathrm{d}t+\varepsilon_t\mathrm{d}\mathbf{w}\quad\text{(rectified flow)},
\end{equation}
\begin{equation}
\label{eq:sde_ddpm}
    \mathrm{d}\mathbf{z}_t=(f_t\mathbf{z}_t-\tfrac{1+\varepsilon_t^2}{2}g_t^2\nabla\log p_t(\mathbf{z}_t))\mathrm{d}t+\varepsilon_t g_t\mathrm{d}\mathbf{w}\quad\text{(diffusion)},
\end{equation}
where $\varepsilon_t$ controls stochasticity, $f_t,g_t$ are schedule-derived coefficients, and $\varepsilon_t{=}0$ recovers the ODE.

\subsection{GRPO Policy Loss and Advantage}

GRPO extends the standard PPO objective to a group-based formulation that eliminates the need for a value network. For each prompt, $G$ candidate generations are sampled from the current policy. Rewards $\{r_1,\ldots,r_G\}$ are collected from a reward model, and the advantage for the $i$-th sample is computed via within-group normalization:
\begin{equation}
\label{eq:advantage}
    A_i = \frac{r_i - \mu_{\text{group}}}{\sigma_{\text{group}} + \epsilon},
\end{equation}
where $\mu_{\text{group}}$ and $\sigma_{\text{group}}$ are the mean and standard deviation of rewards within the group, and $\epsilon$ is a small constant for numerical stability. This can also be extended to global advantage normalization across distributed batches~\cite{wang2026prismprealignmentblackboxonpolicy}.

For diffusion models, the generation process spans $T$ timesteps, and the policy is defined at each denoising step. Let $\pi_\theta(\mathbf{z}_{t-1}\mid\mathbf{z}_t)$ denote the stochastic transition from $\mathbf{z}_t$ to $\mathbf{z}_{t-1}$, parameterized by the SDE in Eq.~\eqref{eq:sde_rf} or~\eqref{eq:sde_ddpm}. The log-probability of a sampled trajectory is the sum of per-step log-probabilities under the Gaussian transition kernel. The GRPO objective with PPO-style clipping for a group of $W$ timesteps is:
\begin{equation}
\label{eq:grpo_loss}
    \mathcal{L}_{\text{GRPO}} = \frac{1}{W}\sum_{j=0}^{W-1} \max\left(-A_i\cdot\rho_j,\; -A_i\cdot\operatorname{clip}(\rho_j, 1-\epsilon_c, 1+\epsilon_c)\right),
\end{equation}
where $\rho_j = \frac{\pi_\theta(\mathbf{z}_{j+1}\mid\mathbf{z}_j)}{\pi_{\text{old}}(\mathbf{z}_{j+1}\mid\mathbf{z}_j)}$ is the probability ratio, $\epsilon_c$ is the clipping range (typically $0.2$), and $A_i$ is the advantage for sample $i$. The per-step log-probability under the SDE is computed as:
\begin{equation}
\label{eq:logprob}
    \log\pi_\theta(\mathbf{z}_{j+1}\mid\mathbf{z}_j) = -\frac{\|\mathbf{z}_{j+1} - \mu_\theta(\mathbf{z}_j, t_j)\|^2}{2\sigma_{t_j}^2} - \log(\sqrt{2\pi}\,\sigma_{t_j}),
\end{equation}
where $\mu_\theta(\mathbf{z}_j, t_j)$ is the predicted mean of the transition and $\sigma_{t_j}$ is the step-dependent standard deviation (the SDE diffusion coefficient at step $j$). Concretely, given model output $\hat{v}_j = v_\theta(\mathbf{z}_j, t_j)$, the mean is obtained by a single Euler step of the corresponding SDE: for rectified flow, $\mu_\theta(\mathbf{z}_j, t_j) = \mathbf{z}_j + \hat{v}_j\cdot(t_{j+1}-t_j)$; for $\boldsymbol{\epsilon}$-prediction, $\mu_\theta(\mathbf{z}_j,t_j) = \alpha_{t_{j+1}}(\mathbf{z}_j/\alpha_{t_j} - (\sigma_{t_j}/\alpha_{t_j}-\sigma_{t_{j+1}}/\alpha_{t_{j+1}})\hat{\boldsymbol{\epsilon}}_j)$. In both cases, $\log\pi_\theta$ is a differentiable function of the model output $\hat{v}_j$ (or $\hat{\boldsymbol{\epsilon}}_j$), which enables computing $s_j = \partial L_j/\partial \hat{v}_j$ via automatic differentiation.

\subsection{JAGG: Jacobian-Aggregated Group Gradient}

The key computational bottleneck in diffusion GRPO training is that the loss in Eq.~\eqref{eq:grpo_loss} requires computing $\partial\mathcal{L}_{\text{GRPO}}/\partial\theta$ through each of the $W$ timesteps. Each timestep requires a separate forward+backward pass through the denoising network, yielding $W$ backward passes per training group.

Our core insight is that the Jacobian $J(z,t) \equiv \partial v_\theta(z,t)/\partial\theta$ of the denoising velocity with respect to model parameters exhibits approximate linearity in the inputs $(z,t)$ during the high-noise denoising regime. Consequently, the Jacobian at intermediate timesteps can be accurately approximated by linear interpolation from the endpoint Jacobians $J_0 = J(z_0, t_0)$ and $J_{W-1} = J(z_{W-1}, t_{W-1})$.

\textbf{Algorithm.} JAGG processes a group of $W$ consecutive denoising timesteps through four phases:
\begin{figure}[ht]
    \centering
    \includegraphics[width=1.2\linewidth]{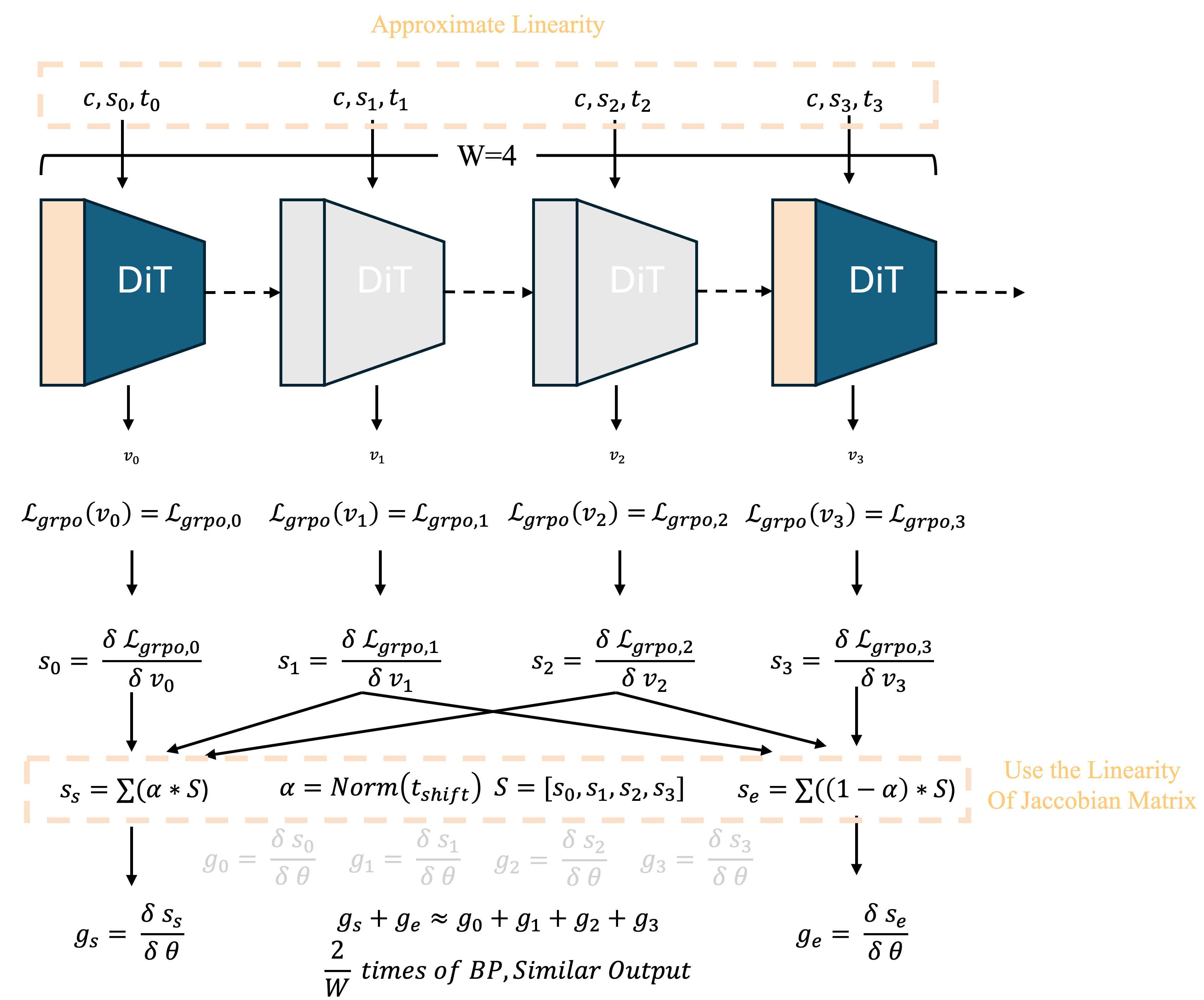}
    \caption{Overview of JAGG. A group of $W=4$ consecutive denoising steps are forward-passed through the DiT. Only the two endpoint steps (colored) maintain gradient graphs; intermediate steps use \texttt{no\_grad}. Upstream gradients $s_j = \partial L_j / \partial \hat{v}_j$ are computed cheaply via the lightweight SDE/PPO objective and then aggregated into $s_s$ and $s_e$ using timestep-normalized weights $\alpha$. Finally, a single joint backward pass from the two endpoints approximates the full $W$-step gradient at $2/W$ the cost.}
    \label{fig:jagg_overview}
\end{figure}
\emph{Phase 1 --- Forward passes.} We run forward passes for all $W$ steps. Only the two endpoints (step $0$ and step $W-1$) are executed with gradient tracking to build computation graphs. The middle $W-2$ steps use $\texttt{torch.no\_grad()}$ for cheap forward evaluation:
\[
    \hat{v}_j = \begin{cases}
        \text{DiT}(z_j, t_j),\; \text{with grad}, & j \in \{0, W-1\} \\
        \text{DiT}(z_j, t_j),\; \text{no\_grad}, & j \in \{1,\ldots,W-2\}.
    \end{cases}
\]

\emph{Phase 2 --- Upstream gradient computation.} For each step $j$, we create a detached proxy $\tilde{v}_j = \hat{v}_j.\texttt{detach}().\texttt{requires\_grad\_}()$ and compute the per-step loss $L_j$ and its upstream gradient $s_j = \partial L_j / \partial \tilde{v}_j$ via automatic differentiation on the proxy. Crucially, this step involves \emph{no transformer backward pass}---the gradient is taken through only the lightweight SDE step (Eq.~\eqref{eq:sde_rf}) and the PPO clipped objective (Eq.~\eqref{eq:grpo_loss}).

Here $\mathbf{z}_j$ is treated as a \emph{fixed sample} carrying no gradient, so $L_j$ is a purely local function of $\tilde{v}_j$ (the per-step REINFORCE approximation that truncates cross-step Markov dependencies). The \emph{exact baseline} in our ablation applies the same detach strategy, so the comparison isolates the Jacobian interpolation approximation rather than any difference in gradient decoupling.

\noindent\textbf{Remark on PPO clipping compatibility.} The PPO-clip objective sets the surrogate loss to a constant whenever the probability ratio $r_j = \pi_\theta/\pi_{\theta_\text{old}}$ falls outside $[1-\varepsilon, 1+\varepsilon]$, making $\partial L_j/\partial \tilde{v}_j = 0$ (i.e., $s_j = \mathbf{0}$) for those clipped steps. Because $s_j$ enters the aggregated signals $s_{\text{base}}$ and $s_{\text{corr}}$ only as a linear summand (Eq.~\eqref{eq:s_base_corr}), a zero $s_j$ contributes nothing to either endpoint backward in Phase 4---exactly reproducing the effect of the original per-step clip, which would also produce a zero weight-gradient at that step. JAGG therefore inherits the PPO clipping semantics without any special handling: clipped steps are automatically masked out through the zero upstream gradient, and the remaining unclipped steps drive the parameter update.

\emph{Phase 3 --- Timestep-weighted aggregation.} The upstream gradients $s_j$ are aggregated into two composite signals using timestep-based interpolation weights. Let $t_j$ denote the (possibly shift-adjusted) timestep fed into the network at step $j$. The interpolation coefficient is:
\begin{equation}
\label{eq:alpha}
    \alpha_j = \frac{t_0 - t_j}{t_0 - t_{W-1}},
\end{equation}
with $\alpha_j \in [0,1]$ monotonic along the denoising trajectory. Using $t_j$---the actual input coordinate of $v_\theta(z,t)$---rather than $\sigma_j$ keeps the interpolation parameter affine in the same coordinate that Theorem~\ref{thm:linear_interp} assumes, so the choice remains exact under timestep shift schedules where $\sigma_t$ is a nonlinear function of $t$. The aggregated signals are:
\begin{equation}
\label{eq:s_base_corr}
    s_{\text{base}} = \sum_{j=0}^{W-1} (1-\alpha_j)\, s_j,\qquad
    s_{\text{corr}} = \sum_{j=0}^{W-1} \alpha_j\, s_j.
\end{equation}

\emph{Phase 4 --- Joint backward.} A single $\texttt{torch.autograd.backward}$ call is issued on $[\hat{v}_0, \hat{v}_{W-1}]$ with gradients $[s_{\text{base}}, s_{\text{corr}}]$. This triggers exactly one FSDP all-gather per module, avoiding the double all-gather issue that would arise from two sequential $\texttt{.backward()}$ calls.

\textbf{Computational cost.} The cost comparison is summarized below:

\vspace{4pt}
\begin{center}
\begin{tabular}{lcc}
    \hline
    \textbf{Method} & \textbf{Transformer fwd} & \textbf{Transformer bwd} \\
    \hline
    Standard per-step & $W$ & $W$ \\
    JAGG (ours) & $W$ & $\mathbf{2}$ (1 joint call) \\
    \hline
\end{tabular}
\end{center}
\vspace{4pt}

For $W=4$, JAGG provides a $\sim$2$\times$ backward speedup; larger $W$ yields proportionally greater gains. In the full ``graph-reuse'' implementation, the total transformer passes drop from $2W$ to $W+1$, as the $W-2$ middle forward passes are no-grad. Because each transformer backward is inherently $\sim$2$\times$ more expensive than a forward pass and triggers a collective communication under data parallelism, this $W{:}2$ reduction translates into a substantial wall-clock improvement (detailed in Appendix~\ref{app:why-bwd-speedup}).

\textbf{High-noise vs.\ low-noise scheduling.} The linearity assumption is most accurate in the early, high-noise regime. We therefore introduce a \texttt{jagg\_frac} parameter that applies JAGG to the first fraction of denoising groups and uses exact per-step backward for the rest. Its value is derived empirically from the gradient quality analysis in the ``\nameref{sec:empirical-quality}'' section: for T2I models (Flux, QwenImage) we use \texttt{jagg\_frac}$=0.6$.

\subsection{Theoretical Analysis}

We now formally prove that under linear velocity conditions, the $t$-weighted Jacobian interpolation is exact. This provides the theoretical foundation for JAGG's approximation.

\begin{theorem}[Exact Jacobian Interpolation]
\label{thm:linear_interp}
Assume the velocity field $v_\theta(z,t)$ is linear in the inputs $(z,t)$, i.e., all second-order and higher partial derivatives with respect to $(z,t)$ vanish. Let $\{(z_j, t_j)\}_{j=0}^{W-1}$ lie on an equidistant interpolation path between $(z_0, t_0)$ and $(z_{W-1}, t_{W-1})$ such that $\Delta z_j = \alpha_j \Delta z_{W-1}$ and $\Delta t_j = \alpha_j \Delta t_{W-1}$ with $\alpha_j = j/(W-1)$. Then the Jacobian $J_j = \partial v_\theta(z_j,t_j)/\partial\theta$ satisfies:
\begin{equation}
\label{eq:theorem}
    J_j = (1-\alpha_j)J_0 + \alpha_j J_{W-1}.
\end{equation}
\end{theorem}

\begin{proof}
Since $v_\theta$ is linear in $(z,t)$, its Taylor expansion truncates exactly at first order:
\begin{equation}
    v_\theta(z_j, t_j) = v_\theta(z_0, t_0) + A_z\cdot\Delta z_j + A_t\cdot\Delta t_j,
\end{equation}
where $A_z = \partial v_\theta/\partial z|_0$ and $A_t = \partial v_\theta/\partial t|_0$ are the partial derivatives with respect to the inputs $(z,t)$, not $\theta$. Taking the derivative with respect to $\theta$ on both sides:
\begin{equation}
    J_j = J_0 + \frac{\partial A_z}{\partial\theta}\cdot\Delta z_j + \frac{\partial A_t}{\partial\theta}\cdot\Delta t_j.
\end{equation}
Under the equidistant path condition, $\Delta z_j = \alpha_j \Delta z_{W-1}$ and $\Delta t_j = \alpha_j \Delta t_{W-1}$. Substituting:
\begin{align}
    J_j &= J_0 + \alpha_j\left(\frac{\partial A_z}{\partial\theta}\cdot\Delta z_{W-1} + \frac{\partial A_t}{\partial\theta}\cdot\Delta t_{W-1}\right)\\
    &= J_0 + \alpha_j(J_{W-1} - J_0) \\
    &= (1-\alpha_j)J_0 + \alpha_j J_{W-1}. \qedhere
\end{align}
\end{proof}

\textbf{Bridging the assumption to real diffusion trajectories.} The equidistant-path assumption in Theorem~\ref{thm:linear_interp} is not as restrictive as it may first appear in our setting. For rectified flow, the forward marginal $z_t = (1-t)x + t\boldsymbol{\epsilon}$ is itself \emph{exactly affine} in $t$, so along a denoising rollout that is well-aligned with the learned vector field the joint $(z_j, t_j)$ is close to the affine path between the two grouped endpoints---particularly in the early, high-noise regime where the trajectory has not yet committed to fine semantic structure. The $t$-weighted choice in Eq.~\eqref{eq:alpha} is precisely the coordinate in which both this geometric assumption and the linearity assumption are most natural. The ``\nameref{sec:empirical-quality}'' section gives the direct empirical test: we measure the gap between the JAGG-aggregated gradient and the exact per-step gradient on real Flux and QwenImage rollouts and find that the assumption holds well in the high-noise groups (cosine $\geq 0.7$) and degrades smoothly thereafter. The \texttt{jagg\_frac} routing rule is the engineering response to this graceful degradation: JAGG is deployed exactly on the groups where Theorem~\ref{thm:linear_interp}'s premise is approximately satisfied.

\textbf{Remark on timestep-shift schedules.} Modern flow-matching backbones (Flux, QwenImage) feed a \emph{shifted} timestep $t' = \mathrm{shift}(t)$ into the network, in which case $\sigma_t$ becomes a nonlinear function of $t'$. Defining $\alpha_j$ with respect to the actual network input $t_j$ (rather than $\sigma_j$) keeps the interpolation parameter affine in the same coordinate that Theorem~\ref{thm:linear_interp} assumes, so the exactness statement holds uniformly across schedules. For a uniform $t$-spacing without shift, $\alpha_j$ reduces to $j/(W-1)$, recovering the classical equidistant case.

\textbf{Corollary (gradient exactness).} Substituting Eq.~\eqref{eq:theorem} into the exact per-step gradient $\sum_{j=0}^{W-1} J_j^{\top} s_j$ and regrouping by the definitions of $s_{\text{base}}, s_{\text{corr}}$ in Eq.~\eqref{eq:s_base_corr} yields
\begin{equation}
    \nabla_\theta\mathcal{L}_{\text{GRPO}} = J_0^{\top}\, s_{\text{base}} + J_{W-1}^{\top}\, s_{\text{corr}},
\end{equation}
i.e., the two-endpoint backward pass produced by JAGG returns the \emph{exact} group gradient whenever the assumption of Theorem~\ref{thm:linear_interp} holds.

\subsection{Memory-Efficient Implementation Variants}
\label{sec:impl}

The joint-backward variant described above (Algorithm~\ref{alg:jagg-joint}) builds computation graphs for both endpoints simultaneously and issues a single \texttt{torch.autograd.backward} call, which is optimal for FSDP communication. However, for very large models with billions of parameters, two live computation graphs may exceed GPU memory budget. We therefore provide a second \textbf{sequential-graph} variant (Algorithm~\ref{alg:jagg-seq}) that processes the two endpoint graphs one at a time:

\begin{enumerate}
  \item Run no-grad forward for the middle $W{-}2$ steps and accumulate partial $s_{\text{base}}$/$s_{\text{corr}}$ contributions.
  \item Run with-grad forward for $j{=}0$; compute $s_0$ via proxy and add to $s_{\text{base}}$; call $\hat{v}_0.\texttt{backward}(s_{\text{base}})$ immediately. The graph for step 0 is freed before step $W{-}1$ is ever built.
  \item Run with-grad forward for $j{=}W{-}1$; compute $s_{W-1}$ via proxy and add to $s_{\text{corr}}$; call $\hat{v}_{W-1}.\texttt{backward}(s_{\text{corr}})$.
\end{enumerate}

The sequential variant keeps at most \emph{one} computation graph in memory at any time, at the cost of two separate FSDP all-gathers (one per backward). In practice we use the joint variant for T2I models, reserving the sequential variant for cases where peak memory is critical. Both variants produce identical gradients and achieve the same $\sim$2$\times$ backward speedup. The full pseudocode for both variants is provided in the appendix (Algorithms~\ref{alg:jagg-joint} and~\ref{alg:jagg-seq}).

\section{Experiments}

We evaluate JAGG on text-to-image (T2I) benchmarks, comparing against the standard per-step GRPO baseline and an ablation variant (\textbf{Ablate}). The Ablate variant performs backward passes only on the first ($j{=}0$) and last ($j{=}W{-}1$) velocity predictions within each group---identical to JAGG in backward count---but assigns the full loss signal to each endpoint independently, without any Jacobian interpolation or aggregation of the intermediate steps. Concretely, it computes $s_0$ and $s_{W-1}$ only (discarding $s_1,\ldots,s_{W-2}$) and issues a joint backward with equal weighting, rather than using $t$-weighted aggregation. The same \texttt{jagg\_frac} schedule is applied, making this a controlled ablation that isolates the contribution of the two-endpoint Jacobian interpolation. All experiments use a window size of $W=4$ and \texttt{jagg\_frac}$=0.6$ unless noted. We report HPSv2.1~\cite{wu2023humanpreferencescorev2}, PickScore~\cite{kirstain2023pickapic}, and CLIPScore~\cite{hessel2022clipscorereferencefreeevaluationmetric} as evaluation metrics, along with wall-clock step time to quantify training efficiency.

\noindent\textbf{Setup.} T2I models (QwenImage, Flux) are trained on \textbf{8 GPUs} each with \textbf{270\,GB} of memory and connected via high-speed interconnect. All inference is performed on the same \textbf{8 GPUs}. T2I quality metrics are computed on the \textbf{official HPSv2.1 test set}~\cite{wu2023humanpreferencescorev2}.

\subsection{Text-to-Image Generation based on DanceGRPO}

Based on DanceGRPO~\cite{xue2025dancegrpounleashinggrpovisual}, we first conduct a lightweight smoke test on QwenImage~\cite{wu2025qwenimagetechnicalreport} using only the HPSv2.1 reward, checking whether JAGG reproduces the reward-improvement trend of standard GRPO before moving to the mixed-reward experiments.

\noindent\textbf{QwenImage HPSv2.1-only smoke test.} Table~\ref{tab:qwen_hps} reports HPSv2.1 scores at evaluation checkpoints (reward curve in Figure~\ref{fig:qwen_hps_smoke}, ``\nameref{sec:appendix-exp-figs}'' appendix). JAGG closely tracks the baseline.

\begin{table}[htbp]
\centering
\small
\caption{QwenImage smoke test (HPSv2.1 reward only). Bold = best per step.}
\label{tab:qwen_hps}
\begin{tabular}{lcc}
\hline
\textbf{Step} & \textbf{Baseline} & \textbf{JAGG (ours)} \\
\hline
0  & 0.3199 & 0.3199 \\
25 & 0.3199 & 0.3199 \\
50 & 0.3447 & 0.3354 \\
75 & 0.3462 & \textbf{0.3463} \\
\hline
\end{tabular}
\end{table}

\noindent\textbf{QwenImage mixed rewards.} We then train QwenImage with GRPO using HPSv2.1 and CLIPScore rewards jointly (1:1 ratio). Figure~\ref{fig:qwen_mixed_stepwise} (deferred to the ``\nameref{sec:appendix-exp-figs}'' appendix) shows the mixed-reward training progression. JAGG reduces the mean per-step wall-clock time from \textbf{835\,s} (baseline) to \textbf{691\,s}---a \textbf{17.3\% speedup}.

Per-checkpoint quantitative results for the QwenImage mixed-reward setting are summarized in Table~\ref{tab:qwen_mixed} (deferred to the ``\nameref{sec:appendix-exp-figs}'' appendix). Figure~\ref{fig:qwen_mixed_reward_mean} plots the mean of HPSv2.1, PickScore, and CLIPScore on the evaluation set over training. JAGG closely tracks the baseline throughout training and matches or exceeds it at convergence. The ablation variant consistently underperforms, confirming the necessity of two-endpoint Jacobian aggregation.

\begin{figure}[h]
    \centering
    \includegraphics[width=0.75\linewidth]{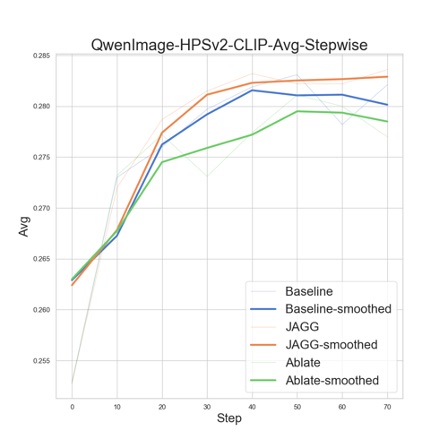}
    \caption{QwenImage trained on mixed HPSv2.1+CLIPScore rewards. Mean of HPSv2.1, PickScore, and CLIPScore on the evaluation set vs.\ training steps for Baseline and JAGG. Per-metric numerical results at each checkpoint are reported in Table~\ref{tab:qwen_mixed} in the ``\nameref{sec:appendix-exp-figs}'' appendix.}
    \label{fig:qwen_mixed_reward_mean}
\end{figure}

\noindent\textbf{Flux mixed rewards.} We additionally evaluate on the Flux flow matching model~\cite{flux2024,labs2025flux1kontextflowmatching} with the same HPSv2.1+CLIPScore reward setup. Figure~\ref{fig:flux_mixed_dancegrpo} (deferred to the ``\nameref{sec:appendix-exp-figs}'' appendix) shows the DanceGRPO mixed-reward training curve for Flux. JAGG closely tracks the baseline trajectory, confirming that our method generalizes across different diffusion architectures.

Numerical results for Flux are reported in Table~\ref{tab:flux_main}. Both methods converge to comparable reward levels, with JAGG slightly ahead on HPSv2.1 and PickScore at step 80. In terms of training efficiency, JAGG reduces the mean per-step wall-clock time from \textbf{279\,s} (baseline) to \textbf{228\,s}---an \textbf{18.2\% speedup}.

\begin{table}[h]
\centering
\small
\caption{Flux evaluation (HPSv2.1+CLIPScore training) at step 80.}
\label{tab:flux_main}
\begin{tabular}{lccc}
\hline
\textbf{Method} & \textbf{HPSv2.1} & \textbf{PickScore} & \textbf{CLIPScore} \\
\hline
Baseline (GRPO) & 0.3157 & 0.2252 & 0.2724 \\
JAGG (ours) & \textbf{0.3216} & \textbf{0.2254} & 0.2712 \\
\hline
\end{tabular}
\end{table}

\subsection{Efficiency Analysis on DanceGRPO}

Under DanceGRPO, JAGG reduces the mean per-step wall-clock time from \textbf{279\,s $\rightarrow$ 228\,s} on Flux (\textbf{18.2\%} speedup) and from \textbf{835\,s $\rightarrow$ 691\,s} on QwenImage (\textbf{17.3\%} speedup). The core efficiency gain of JAGG is the reduction of transformer backward passes from $W$ per group to $2$; for $W=4$ this yields a $\sim$2$\times$ backward-pass reduction. Two factors further amplify each backward pass beyond its raw FLOPs---(i) the backward of a transformer layer costs roughly $2\times$ its forward, and (ii) every backward triggers an all-reduce/reduce-scatter over the full parameter gradient tensor under DP/FSDP. A detailed derivation of these two effects (attention-backward FLOP accounting via Eqs.~\eqref{eq:bwd_V}--\eqref{eq:bwd_QK} and the communication analysis) is provided in Appendix~\ref{app:why-bwd-speedup}.

\subsection{Text-to-Image Generation based on Flow-GRPO-Fast}

To validate the robustness of our method accross different RL Algorithms for DiT models, we also apply JAGG to the Flow-GRPO-Fast algorithm~\cite{liu2025flowgrpotrainingflowmatching}. We trained different recipe on QwenImage~\cite{wu2025qwenimagetechnicalreport} and Flux1.dev~\cite{liu2025flowgrpotrainingflowmatching} models with HPSv2.1 and CLIPScore rewards (1:1 ratio). In this experiment, we run our ablation method on Flux1.dev. 

The training recipe is almost the same as the ones in DanceGRPO. Figure \ref{fig:flow_fast_run_stepwise} is the reward curve of the training of two models.

Figure~\ref{fig:flow_fast_flux_eval} reports the evaluation results measured every 10 training steps for Flux1.dev, and Figure~\ref{fig:flow_fast_qwenimage_eval} shows the corresponding results for QwenImage.

\begin{figure}[htbp]
    \centering
    \includegraphics[width=1\linewidth]{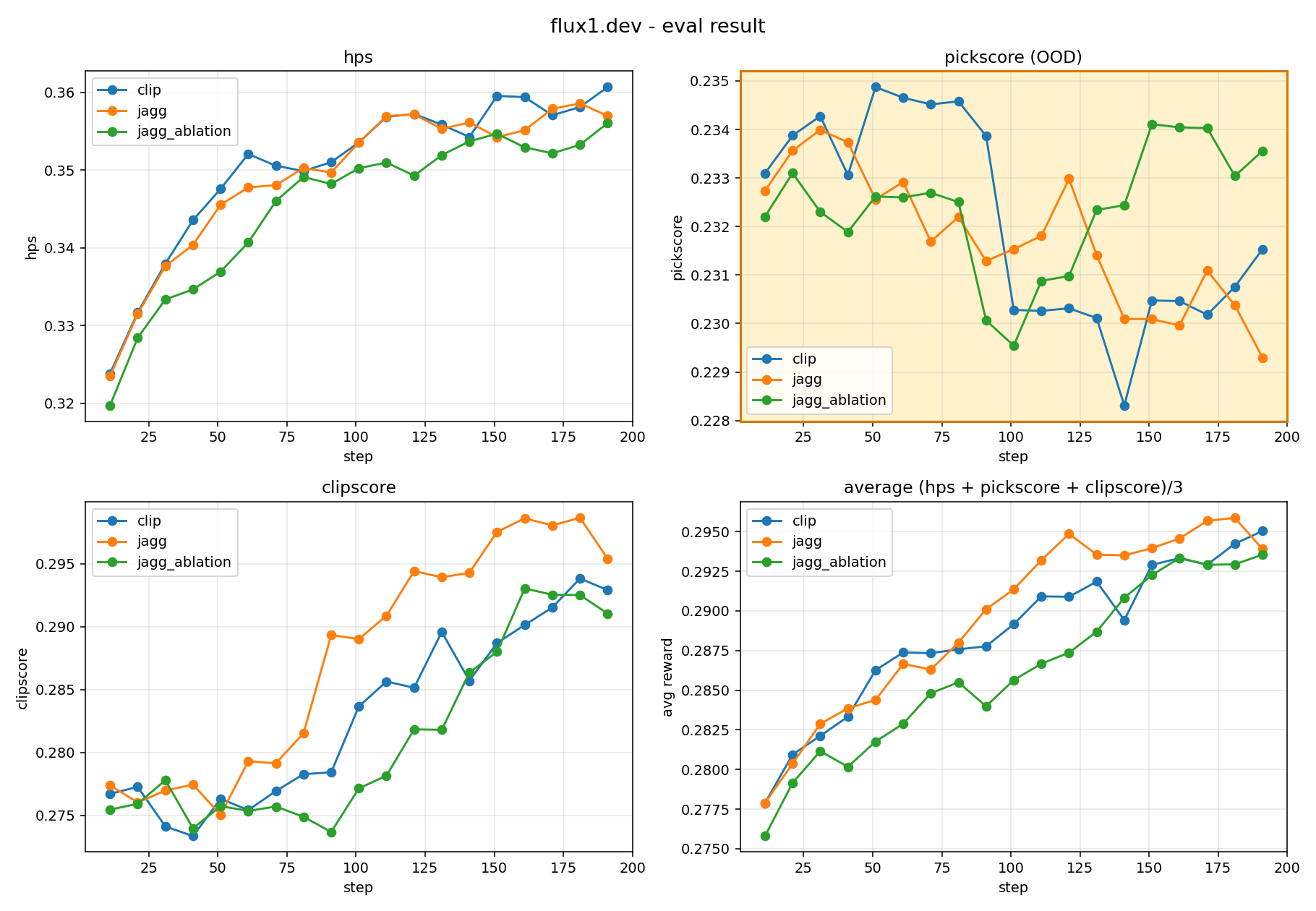}
    \caption{Flux1.dev validate reward curve trained using Flow-GRPO-Fast algorithm under different settings (w JAGG, w/o JAGG, Ablation).}
    \label{fig:flow_fast_flux_eval}
\end{figure}

\begin{figure}[htbp]
    \centering
    \includegraphics[width=1\linewidth]{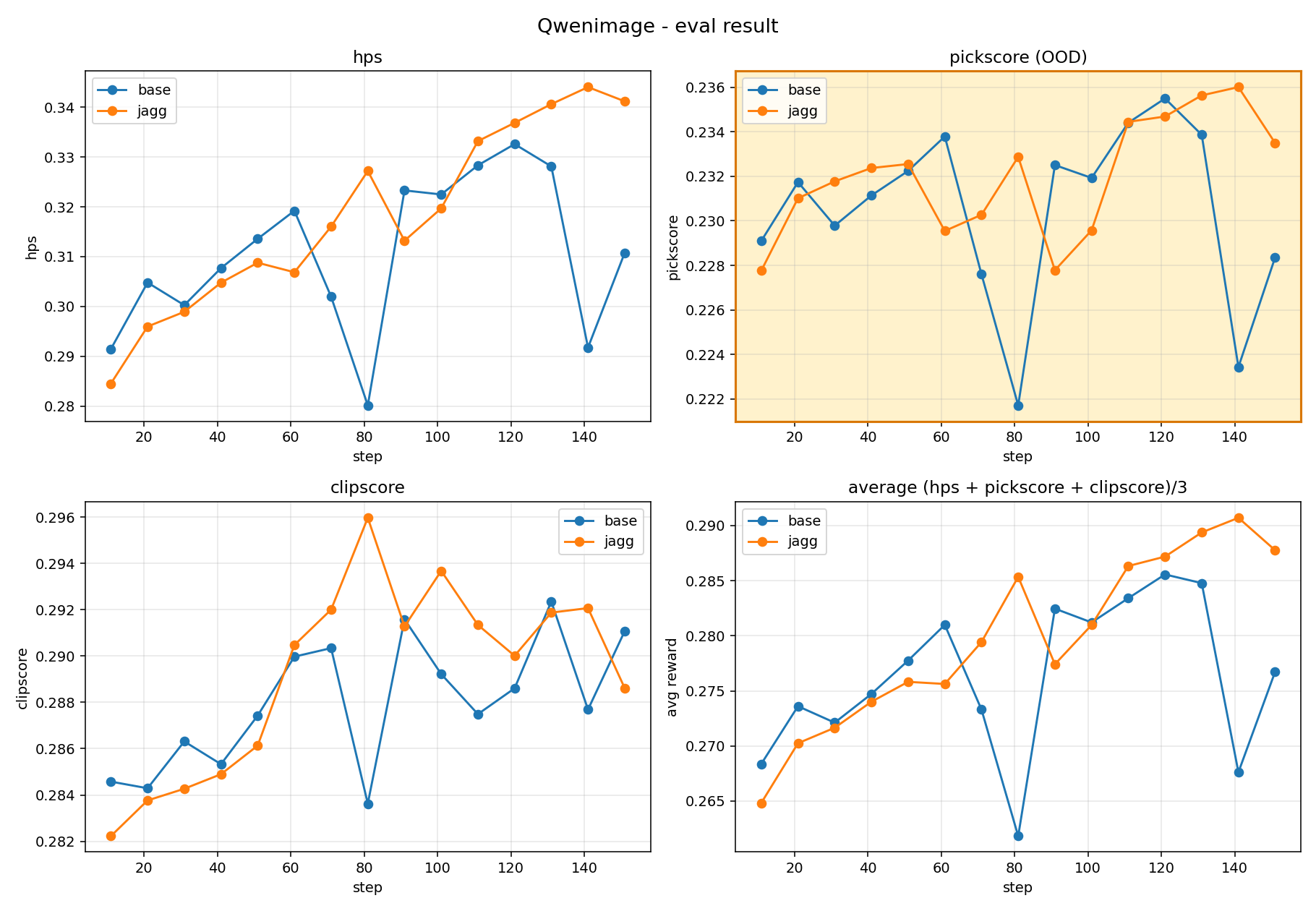}
    \caption{QwenImage validate reward curve trained using Flow-GRPO-Fast algorithm under different settings (w JAGG, w/o JAGG, Ablation).}
    \label{fig:flow_fast_qwenimage_eval}
\end{figure}

Overall, the Flow-GRPO-Fast experiments show that JAGG preserves the reward optimization behavior of the original algorithm while improving the effectiveness of the policy update. For Flux1.dev, JAGG closely follows the non-accelerated baseline on HPSv2.1 and achieves a clear advantage on CLIPScore in the middle and later stages of training, which leads to the best average reward over most checkpoints. The ablation variant is consistently weaker than full JAGG, especially on the averaged metric, indicating that simply back-propagating through two endpoints is insufficient; the intermediate-step gradient information recovered by Jacobian aggregation is important for stable reward improvement. PickScore exhibits larger fluctuations because it is evaluated in an out-of-distribution setting, but JAGG remains comparable to the baseline and does not introduce systematic degradation.

For QwenImage, the benefit of JAGG is even more pronounced in the later training stage. Although the baseline initially improves faster, its HPSv2.1, PickScore, and CLIPScore curves start to decline after the middle checkpoints, suggesting a less stable optimization trajectory. In contrast, JAGG continues to improve or remain stable after step 80 and eventually obtains substantially higher HPSv2.1, PickScore, and average reward. This trend suggests that JAGG is not merely a computational shortcut: by aggregating the gradient signals across the denoising window, it can produce a smoother and more reliable update direction under Flow-GRPO-Fast as well. Together with the Flux1.dev results, these curves demonstrate that JAGG generalizes beyond DanceGRPO and remains effective across different DiT backbones and RL training recipes.

\subsection{Efficiency Analysis on Flow-GRPO-Fast}

We measure the wall-clock cost of one training iteration under the Flow-GRPO-Fast recipe on both QwenImage and Flux1.dev. Figure~\ref{fig:efficiency_flowfast} reports both the end-to-end mean per-step time (left) and the per-step \emph{update} time (right). Here ``update'' refers to the full policy-update stage that JAGG targets: it consists of (i) a forward pass on the stored experience (rolled-out latents and old log-probabilities) to reconstruct the current-policy log-probabilities and the surrogate loss, and (ii) the backward pass that produces the parameter gradients. JAGG only alters this update stage---the rollout, reward evaluation, optimizer step, and communication schedule are unchanged.

JAGG reduces the mean step time from \textbf{674\,s $\rightarrow$ 556\,s} on QwenImage (17.5\% speedup) and \textbf{316\,s $\rightarrow$ 272\,s} on Flux1.dev (13.9\% speedup). The update time alone drops from \textbf{393\,s $\rightarrow$ 275\,s} on QwenImage (30.0\% speedup) and \textbf{144\,s $\rightarrow$ 101\,s} on Flux1.dev (29.9\% speedup). The much larger reduction on the update stage relative to the full step confirms that the savings come from replacing the per-timestep forward+backward passes with JAGG's two-endpoint schedule.

\begin{figure}[htbp]
    \centering
    \includegraphics[width=0.48\linewidth]{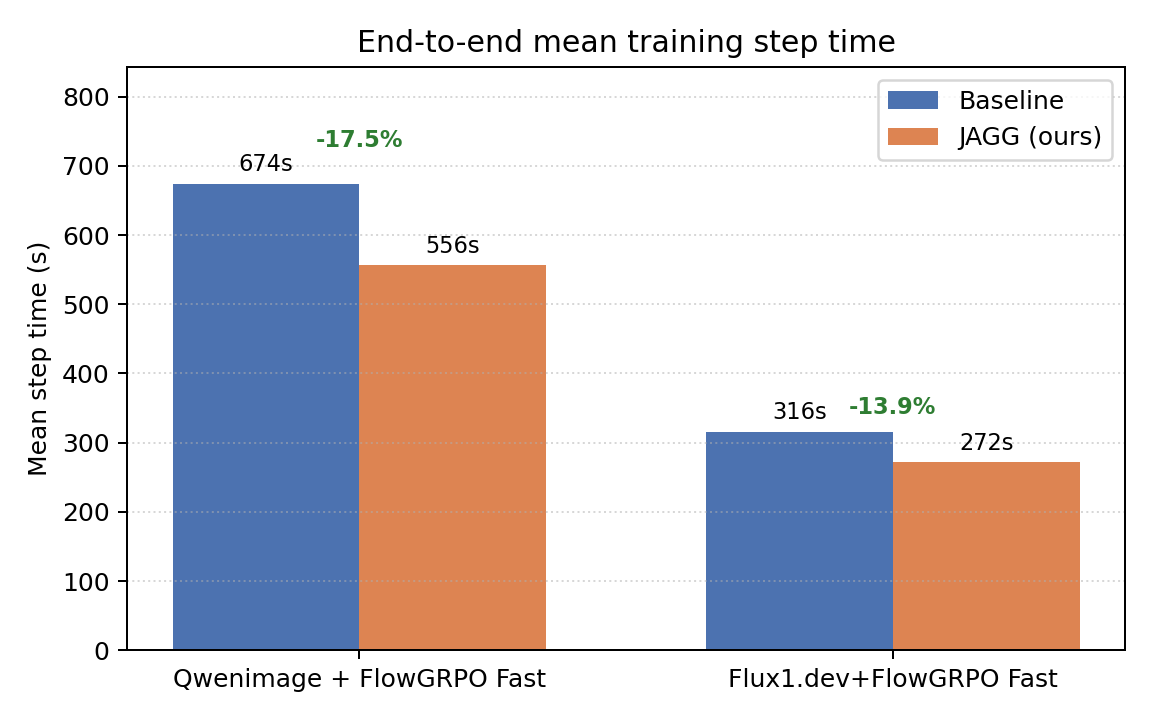}\hfill\includegraphics[width=0.48\linewidth]{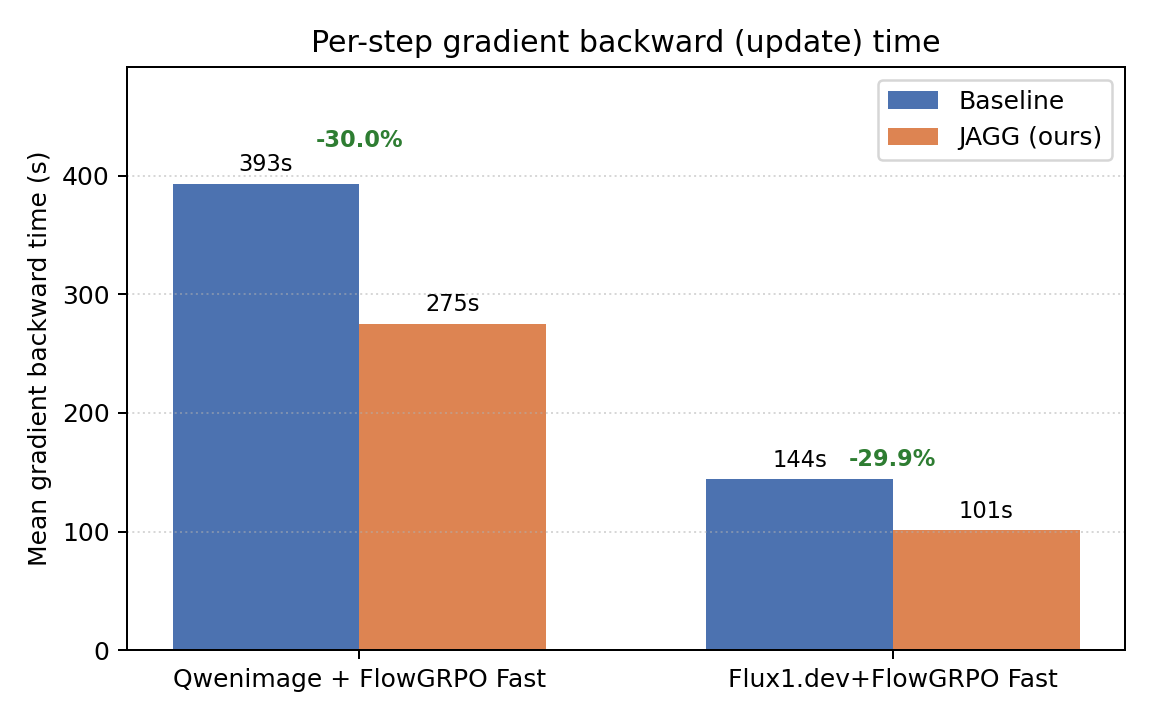}
    \caption{\textbf{Flow-GRPO-Fast + JAGG:} \textbf{Left:} end-to-end mean training step time (seconds). \textbf{Right:} per-step update time (forward on stored experience + backward), which is the stage JAGG directly targets. Percentages above the bars indicate the wall-clock speedup of JAGG over the baseline.}
    \label{fig:efficiency_flowfast}
\end{figure}

The end-to-end wall-clock numbers for Flow-GRPO-Fast in Figure~\ref{fig:efficiency_flowfast} are also consistent with the DanceGRPO speedups reported above; the underlying FLOP and communication analysis is shared and detailed in Appendix~\ref{app:why-bwd-speedup}.

\subsection{Analysis}

\noindent\textbf{Why Two Endpoints Outperforms.}
The ablation results consistently show that JAGG outperforms the endpoint-only variant. This is expected: interpolating intermediate gradients recovers more information than discarding them. Additionally, JAGG's linear interpolation of Jacobians effectively marginalizes out stochastic noise perturbations that cause per-step fluctuations, analogous to replacing an SDE trajectory with its deterministic ODE skeleton~\cite{li2026mixgrpounlockingflowbasedgrpo}, which may explain why JAGG occasionally exceeds the exact baseline.

\noindent\textbf{Effect of \texttt{jagg\_frac}.}
The \texttt{jagg\_frac} hyperparameter controls what fraction of denoising groups use JAGG versus exact per-step backward. We set it via the $\cos \geq 0.5$ threshold from the gradient quality analysis (``\nameref{sec:empirical-quality}'' section): $\texttt{jagg\_frac}=0.6$ for T2I models.

\section{Conclusion}

We presented JAGG, a Jacobian-aggregation framework for accelerating GRPO training of diffusion models. By exploiting the near-linearity of velocity Jacobians in the high-noise regime, JAGG reduces backward passes from $W$ to $2$ per group of $W$ timesteps, achieving consistent wall-clock speedups (17--18\% per training step) with minimal quality degradation. We provided a rigorous proof that the $t$-weighted interpolation is exact under linear conditions, and demonstrated the generality of our method across multiple text-to-image diffusion backbones. JAGG represents a significant step toward making large-scale diffusion GRPO training practical, and we hope it enables more widespread adoption of RL-based alignment for visual generative models.

\FloatBarrier
\bibliography{refs}

\begin{thebibliography}{47}
\providecommand{\natexlab}[1]{#1}

\bibitem[{Albergo, Boffi, and
  Vanden-Eijnden(2025)}]{albergo2025stochasticinterpolantsunifyingframework}
Albergo, M.~S.; Boffi, N.~M.; and Vanden-Eijnden, E. 2025.
\newblock Stochastic Interpolants: A Unifying Framework for Flows and
  Diffusions.
\newblock arXiv:2303.08797.

\bibitem[{Albergo and
  Vanden-Eijnden(2023)}]{albergo2023buildingnormalizingflowsstochastic}
Albergo, M.~S.; and Vanden-Eijnden, E. 2023.
\newblock Building Normalizing Flows with Stochastic Interpolants.
\newblock arXiv:2209.15571.

\bibitem[{Black et~al.(2024)Black, Janner, Du, Kostrikov, and
  Levine}]{black2024trainingdiffusionmodelsreinforcement}
Black, K.; Janner, M.; Du, Y.; Kostrikov, I.; and Levine, S. 2024.
\newblock Training Diffusion Models with Reinforcement Learning.
\newblock arXiv:2305.13301.

\bibitem[{Bottou, Curtis, and
  Nocedal(2018)}]{bottou2018optimizationmethodslargescalemachine}
Bottou, L.; Curtis, F.~E.; and Nocedal, J. 2018.
\newblock Optimization Methods for Large-Scale Machine Learning.
\newblock arXiv:1606.04838.

\bibitem[{Chen et~al.(2024{\natexlab{a}})Chen, Wang, Wu, Liao, Sun, Yan, and
  Lin}]{10.1007/978-3-031-72698-9_11}
Chen, C.; Wang, A.; Wu, H.; Liao, L.; Sun, W.; Yan, Q.; and Lin, W.
  2024{\natexlab{a}}.
\newblock Enhancing Diffusion Models with Text-Encoder Reinforcement Learning.
\newblock In \emph{Computer Vision – ECCV 2024: 18th European Conference,
  Milan, Italy, September 29–October 4, 2024, Proceedings, Part XXV},
  182–198. Berlin, Heidelberg: Springer-Verlag.
\newblock ISBN 978-3-031-72697-2.

\bibitem[{Chen et~al.(2024{\natexlab{b}})Chen, Shen, Ye, Cao, Tu, Bouganis,
  Zhao, and Chen}]{chen2024deltadittrainingfreeaccelerationmethod}
Chen, P.; Shen, M.; Ye, P.; Cao, J.; Tu, C.; Bouganis, C.-S.; Zhao, Y.; and
  Chen, T. 2024{\natexlab{b}}.
\newblock $\Delta$-DiT: A Training-Free Acceleration Method Tailored for
  Diffusion Transformers.
\newblock arXiv:2406.01125.

\bibitem[{Chen et~al.(2025)Chen, Chai, Yang, Zhang, Zhou, Quek, Poria, and
  Liu}]{chen2025diffpodiffusionstyledpreferenceoptimization}
Chen, R.; Chai, W.; Yang, Z.; Zhang, X.; Zhou, J.~T.; Quek, T.; Poria, S.; and
  Liu, Z. 2025.
\newblock DiffPO: Diffusion-styled Preference Optimization for Efficient
  Inference-Time Alignment of Large Language Models.
\newblock arXiv:2503.04240.

\bibitem[{Choi et~al.(2026)Choi, Zhu, Guo, Molodyk, Yuan, Bai, Xin, Tao, and
  Chen}]{choi2026rethinkingdesignspacereinforcement}
Choi, J.; Zhu, Y.; Guo, W.; Molodyk, P.; Yuan, B.; Bai, J.; Xin, Y.; Tao, M.;
  and Chen, Y. 2026.
\newblock Rethinking the Design Space of Reinforcement Learning for Diffusion
  Models: On the Importance of Likelihood Estimation Beyond Loss Design.
\newblock arXiv:2602.04663.

\bibitem[{Du et~al.(2024)Du, Zhang, Liu, Wang, Lin, Li, Niyato, Kang, Xiong,
  Cui, Ai, Zhou, and Kim}]{10529221}
Du, H.; Zhang, R.; Liu, Y.; Wang, J.; Lin, Y.; Li, Z.; Niyato, D.; Kang, J.;
  Xiong, Z.; Cui, S.; Ai, B.; Zhou, H.; and Kim, D.~I. 2024.
\newblock Enhancing Deep Reinforcement Learning: A Tutorial on Generative
  Diffusion Models in Network Optimization.
\newblock \emph{IEEE Communications Surveys \& Tutorials}, 26(4): 2611--2646.

\bibitem[{Fan et~al.(2023)Fan, Watkins, Du, Liu, Ryu, Boutilier, Abbeel,
  Ghavamzadeh, Lee, and Lee}]{NEURIPS2023_fc65fab8}
Fan, Y.; Watkins, O.; Du, Y.; Liu, H.; Ryu, M.; Boutilier, C.; Abbeel, P.;
  Ghavamzadeh, M.; Lee, K.; and Lee, K. 2023.
\newblock DPOK: Reinforcement Learning for Fine-tuning Text-to-Image Diffusion
  Models.
\newblock In Oh, A.; Naumann, T.; Globerson, A.; Saenko, K.; Hardt, M.; and
  Levine, S., eds., \emph{Advances in Neural Information Processing Systems},
  volume~36, 79858--79885. Curran Associates, Inc.

\bibitem[{Gu et~al.(2026)Gu, He, Su, He, Wang, and Liu}]{gu2026scalingcache}
Gu, L.; He, J.; Su, L.; He, K.; Wang, W.; and Liu, Y. 2026.
\newblock ScalingCache: Extreme Acceleration of DiTs through Difference Scaling
  and Dynamic Interval Caching.
\newblock In \emph{The Fourteenth International Conference on Learning
  Representations}.

\bibitem[{Hessel et~al.(2022)Hessel, Holtzman, Forbes, Bras, and
  Choi}]{hessel2022clipscorereferencefreeevaluationmetric}
Hessel, J.; Holtzman, A.; Forbes, M.; Bras, R.~L.; and Choi, Y. 2022.
\newblock CLIPScore: A Reference-free Evaluation Metric for Image Captioning.
\newblock arXiv:2104.08718.

\bibitem[{Ho, Jain, and
  Abbeel(2020)}]{ho2020denoisingdiffusionprobabilisticmodels}
Ho, J.; Jain, A.; and Abbeel, P. 2020.
\newblock Denoising Diffusion Probabilistic Models.
\newblock arXiv:2006.11239.

\bibitem[{Kirstain et~al.(2023)Kirstain, Polyak, Singer, Matiana, Penna, and
  Levy}]{kirstain2023pickapic}
Kirstain, Y.; Polyak, A.; Singer, U.; Matiana, S.; Penna, J.; and Levy, O.
  2023.
\newblock Pick-a-Pic: An Open Dataset of User Preferences for Text-to-Image
  Generation.
\newblock arXiv:2305.01569.

\bibitem[{Labs(2024)}]{flux2024}
Labs, B.~F. 2024.
\newblock FLUX.
\newblock \url{https://github.com/black-forest-labs/flux}.

\bibitem[{Labs et~al.(2025)Labs, Batifol, Blattmann, Boesel, Consul, Diagne,
  Dockhorn, English, English, Esser, Kulal, Lacey, Levi, Li, Lorenz, Müller,
  Podell, Rombach, Saini, Sauer, and Smith}]{labs2025flux1kontextflowmatching}
Labs, B.~F.; Batifol, S.; Blattmann, A.; Boesel, F.; Consul, S.; Diagne, C.;
  Dockhorn, T.; English, J.; English, Z.; Esser, P.; Kulal, S.; Lacey, K.;
  Levi, Y.; Li, C.; Lorenz, D.; Müller, J.; Podell, D.; Rombach, R.; Saini,
  H.; Sauer, A.; and Smith, L. 2025.
\newblock FLUX.1 Kontext: Flow Matching for In-Context Image Generation and
  Editing in Latent Space.
\newblock arXiv:2506.15742.

\bibitem[{Lee and Ye(2026)}]{lee2026pcpoproportionatecreditpolicy}
Lee, J.; and Ye, J.~C. 2026.
\newblock PCPO: Proportionate Credit Policy Optimization for Aligning Image
  Generation Models.
\newblock arXiv:2509.25774.

\bibitem[{Li et~al.(2026)Li, Cui, Huang, Ma, Fan, Cheng, Yang, Zhong, and
  Bo}]{li2026mixgrpounlockingflowbasedgrpo}
Li, J.; Cui, Y.; Huang, T.; Ma, Y.; Fan, C.; Cheng, Y.; Yang, M.; Zhong, Z.;
  and Bo, L. 2026.
\newblock MixGRPO: Unlocking Flow-based GRPO Efficiency with Mixed ODE-SDE.
\newblock arXiv:2507.21802.

\bibitem[{Li et~al.(2025)Li, Wang, Zhu, Zhao, Lu, She, and
  Zhang}]{li2025branchgrpostableefficientgrpo}
Li, Y.; Wang, Y.; Zhu, Y.; Zhao, Z.; Lu, M.; She, Q.; and Zhang, S. 2025.
\newblock BranchGRPO: Stable and Efficient GRPO with Structured Branching in
  Diffusion Models.
\newblock arXiv:2509.06040.

\bibitem[{Liu et~al.(2025{\natexlab{a}})Liu, Liu, Liang, Li, Liu, Wang, Wan,
  Zhang, and Ouyang}]{liu2025flowgrpotrainingflowmatching}
Liu, J.; Liu, G.; Liang, J.; Li, Y.; Liu, J.; Wang, X.; Wan, P.; Zhang, D.; and
  Ouyang, W. 2025{\natexlab{a}}.
\newblock Flow-GRPO: Training Flow Matching Models via Online RL.
\newblock arXiv:2505.05470.

\bibitem[{Liu et~al.(2026)Liu, Ye, Yuan, Zhu, Gao, Wu, Li, Wang, Nie, Huang,
  and Ouyang}]{liu2026unigrpounifiedpolicyoptimization}
Liu, J.; Ye, Z.; Yuan, L.; Zhu, S.; Gao, Y.; Wu, J.; Li, K.; Wang, X.; Nie, X.;
  Huang, W.; and Ouyang, W. 2026.
\newblock UniGRPO: Unified Policy Optimization for Reasoning-Driven Visual
  Generation.
\newblock arXiv:2603.23500.

\bibitem[{Liu et~al.(2025{\natexlab{b}})Liu, Zou, Lyu, Chen, and
  Zhang}]{liu2025reusingforecastingacceleratingdiffusion}
Liu, J.; Zou, C.; Lyu, Y.; Chen, J.; and Zhang, L. 2025{\natexlab{b}}.
\newblock From Reusing to Forecasting: Accelerating Diffusion Models with
  TaylorSeers.
\newblock arXiv:2503.06923.

\bibitem[{Liu, Gong, and Liu(2022)}]{liu2022rectifiedflowstraightfastlearning}
Liu, X.; Gong, C.; and Liu, Q. 2022.
\newblock Flow Straight and Fast: Learning to Generate and Transfer Data with
  Rectified Flow.
\newblock arXiv:2209.03003.

\bibitem[{Luo et~al.(2026)Luo, Hu, Fan, Sun, Chen, Xia, Zhang, Chang, and
  Wang}]{luo2026reinforcementlearningmeetsmasked}
Luo, Y.; Hu, X.; Fan, K.; Sun, H.; Chen, Z.; Xia, B.; Zhang, T.; Chang, Y.; and
  Wang, X. 2026.
\newblock Reinforcement Learning Meets Masked Generative Models: Mask-GRPO for
  Text-to-Image Generation.
\newblock arXiv:2510.13418.

\bibitem[{Ma et~al.(2025)Ma, Zhang, Wang, and
  Ye}]{ma2025consolidatingreinforcementlearningmultimodal}
Ma, T.; Zhang, M.; Wang, Y.; and Ye, Q. 2025.
\newblock Consolidating Reinforcement Learning for Multimodal Discrete
  Diffusion Models.
\newblock arXiv:2510.02880.

\bibitem[{Miao et~al.(2024)Miao, Wang, Wang, Yang, Wang, Qiu, and
  Liu}]{10656815}
Miao, Z.; Wang, J.; Wang, Z.; Yang, Z.; Wang, L.; Qiu, Q.; and Liu, Z. 2024.
\newblock Training Diffusion Models Towards Diverse Image Generation with
  Reinforcement Learning.
\newblock In \emph{2024 IEEE/CVF Conference on Computer Vision and Pattern
  Recognition (CVPR)}, 10844--10853.

\bibitem[{Ping et~al.(2026)Ping, Chen, Yu, Hui, Yan, and
  Chang}]{ping2026longrunleashinglongcontextreasoning}
Ping, B.; Chen, Z.; Yu, Y.; Hui, T.; Yan, J.; and Chang, B. 2026.
\newblock LongR: Unleashing Long-Context Reasoning via Reinforcement Learning
  with Dense Utility Rewards.
\newblock arXiv:2602.05758.

\bibitem[{Rojas et~al.(2026)Rojas, Lin, Rasul, Schneider, Nevmyvaka, Tao, and
  Deng}]{rojas2026improvingreasoningdiffusionlanguage}
Rojas, K.; Lin, J.; Rasul, K.; Schneider, A.; Nevmyvaka, Y.; Tao, M.; and Deng,
  W. 2026.
\newblock Improving Reasoning for Diffusion Language Models via Group Diffusion
  Policy Optimization.
\newblock arXiv:2510.08554.

\bibitem[{Shao et~al.(2024)Shao, Wang, Zhu, Xu, Song, Bi, Zhang, Zhang, Li, Wu,
  and Guo}]{shao2024deepseekmathpushinglimitsmathematical}
Shao, Z.; Wang, P.; Zhu, Q.; Xu, R.; Song, J.; Bi, X.; Zhang, H.; Zhang, M.;
  Li, Y.~K.; Wu, Y.; and Guo, D. 2024.
\newblock DeepSeekMath: Pushing the Limits of Mathematical Reasoning in Open
  Language Models.
\newblock arXiv:2402.03300.

\bibitem[{Song, Meng, and
  Ermon(2022)}]{song2022denoisingdiffusionimplicitmodels}
Song, J.; Meng, C.; and Ermon, S. 2022.
\newblock Denoising Diffusion Implicit Models.
\newblock arXiv:2010.02502.

\bibitem[{Tang et~al.(2026)Tang, Zhang, Wang, Mao, Schmidt, and
  Yeung-Levy}]{tang2026vgrpoonlinereinforcementlearning}
Tang, B.; Zhang, Y.; Wang, X.; Mao, J.; Schmidt, L.; and Yeung-Levy, S. 2026.
\newblock V-GRPO: Online Reinforcement Learning for Denoising Generative Models
  Is Easier than You Think.
\newblock arXiv:2604.23380.

\bibitem[{Wang et~al.(2026{\natexlab{a}})Wang, Deng, Pan, Liu, Wang, Zhang, Qi,
  and Wang}]{wang2026udmgrpostableefficientgroup}
Wang, J.; Deng, H.; Pan, T.; Liu, Y.; Wang, C.; Zhang, F.; Qi, Y.; and Wang, X.
  2026{\natexlab{a}}.
\newblock UDM-GRPO: Stable and Efficient Group Relative Policy Optimization for
  Uniform Discrete Diffusion Models.
\newblock arXiv:2604.18518.

\bibitem[{Wang et~al.(2025)Wang, Liang, Liu, Liu, Liu, Zheng, Pang, Ma, Xie,
  Wang, Wang, Wan, and
  Liang}]{wang2025grpoguardmitigatingimplicitoveroptimization}
Wang, J.; Liang, J.; Liu, J.; Liu, H.; Liu, G.; Zheng, J.; Pang, W.; Ma, A.;
  Xie, Z.; Wang, X.; Wang, M.; Wan, P.; and Liang, X. 2025.
\newblock GRPO-Guard: Mitigating Implicit Over-Optimization in Flow Matching
  via Regulated Clipping.
\newblock arXiv:2510.22319.

\bibitem[{Wang et~al.(2026{\natexlab{b}})Wang, Huang, Yu, Yang, Lin, Wu, Xiao,
  Chen, Wang, Zhu, Zhang, and Qin}]{wang2026prismprealignmentblackboxonpolicy}
Wang, S.; Huang, W.; Yu, X.; Yang, Z.; Lin, H.; Wu, K.; Xiao, C.; Chen, C.;
  Wang, W.; Zhu, B.; Zhang, Y.; and Qin, C. 2026{\natexlab{b}}.
\newblock Beyond SFT-to-RL: Pre-alignment via Black-Box On-Policy Distillation
  for Multimodal RL.
\newblock arXiv:2604.28123.

\bibitem[{Wu et~al.(2025)Wu, Li, Zhou, Lin, Gao, Yan, ming Yin, Bai, Xu, Chen,
  Chen, Tang, Zhang, Wang, Yang, Yu, Cheng, Liu, Li, Zhang, Meng, Wei, Ni,
  Chen, Cao, Peng, Qu, Wu, Wang, Yu, Wen, Feng, Xu, Wang, Zhang, Zhu, Wu, Cai,
  and Liu}]{wu2025qwenimagetechnicalreport}
Wu, C.; Li, J.; Zhou, J.; Lin, J.; Gao, K.; Yan, K.; ming Yin, S.; Bai, S.; Xu,
  X.; Chen, Y.; Chen, Y.; Tang, Z.; Zhang, Z.; Wang, Z.; Yang, A.; Yu, B.;
  Cheng, C.; Liu, D.; Li, D.; Zhang, H.; Meng, H.; Wei, H.; Ni, J.; Chen, K.;
  Cao, K.; Peng, L.; Qu, L.; Wu, M.; Wang, P.; Yu, S.; Wen, T.; Feng, W.; Xu,
  X.; Wang, Y.; Zhang, Y.; Zhu, Y.; Wu, Y.; Cai, Y.; and Liu, Z. 2025.
\newblock Qwen-Image Technical Report.
\newblock arXiv:2508.02324.

\bibitem[{Wu et~al.(2023)Wu, Hao, Sun, Chen, Zhu, Zhao, and
  Li}]{wu2023humanpreferencescorev2}
Wu, X.; Hao, Y.; Sun, K.; Chen, Y.; Zhu, F.; Zhao, R.; and Li, H. 2023.
\newblock Human Preference Score v2: A Solid Benchmark for Evaluating Human
  Preferences of Text-to-Image Synthesis.
\newblock arXiv:2306.09341.

\bibitem[{Xiao et~al.(2024)Xiao, Jiang, Zhou, Zhang, Huang, and
  Yang}]{10987802}
Xiao, Z.; Jiang, P.; Zhou, M.; Zhang, J.; Huang, Z.; and Yang, J. 2024.
\newblock DiffPPO: Reinforcement Learning Fine-Tuning of Diffusion Models for
  Text-to-Image Generation.
\newblock In \emph{2024 International Conference on Neuromorphic Computing
  (ICNC)}, 1--5.

\bibitem[{Xue et~al.(2025)Xue, Wu, Gao, Kong, Zhu, Chen, Liu, Liu, Guo, Huang,
  and Luo}]{xue2025dancegrpounleashinggrpovisual}
Xue, Z.; Wu, J.; Gao, Y.; Kong, F.; Zhu, L.; Chen, M.; Liu, Z.; Liu, W.; Guo,
  Q.; Huang, W.; and Luo, P. 2025.
\newblock DanceGRPO: Unleashing GRPO on Visual Generation.
\newblock arXiv:2505.07818.

\bibitem[{Yuan et~al.(2025)Yuan, Liu, Yue, Zhang, Zuo, Wang, Zhang, and
  Zhou}]{yuan2025argrpotrainingautoregressiveimage}
Yuan, S.; Liu, Y.; Yue, Y.; Zhang, J.; Zuo, W.; Wang, Q.; Zhang, F.; and Zhou,
  G. 2025.
\newblock AR-GRPO: Training Autoregressive Image Generation Models via
  Reinforcement Learning.
\newblock arXiv:2508.06924.

\bibitem[{Zhai et~al.(2026)Zhai, Chen, Zhao, Qian, Shen, and
  Lu}]{zhai2026rewardslabelsrevisitingrlvr}
Zhai, Z.; Chen, M.; Zhao, J.; Qian, J.; Shen, L.; and Lu, Y. 2026.
\newblock Rewards as Labels: Revisiting RLVR from a Classification Perspective.
\newblock arXiv:2602.05630.

\bibitem[{Zhang et~al.(2025)Zhang, Zuo, He, Sun, Liu, Jiang, Fan, Tian, Jia,
  Li, Fu, Lv, Zhang, Zeng, Qu, Li, Wang, Wang, Long, Liu, Xu, Ma, Zhu, Hua,
  Liu, Li, Chen, Qu, Li, Chen, Yuan, Gao, Li, Ma, Cui, Liu, Qi, Ding, and
  Zhou}]{zhang2025surveyreinforcementlearninglarge}
Zhang, K.; Zuo, Y.; He, B.; Sun, Y.; Liu, R.; Jiang, C.; Fan, Y.; Tian, K.;
  Jia, G.; Li, P.; Fu, Y.; Lv, X.; Zhang, Y.; Zeng, S.; Qu, S.; Li, H.; Wang,
  S.; Wang, Y.; Long, X.; Liu, F.; Xu, X.; Ma, J.; Zhu, X.; Hua, E.; Liu, Y.;
  Li, Z.; Chen, H.; Qu, X.; Li, Y.; Chen, W.; Yuan, Z.; Gao, J.; Li, D.; Ma,
  Z.; Cui, G.; Liu, Z.; Qi, B.; Ding, N.; and Zhou, B. 2025.
\newblock A Survey of Reinforcement Learning for Large Reasoning Models.
\newblock arXiv:2509.08827.

\bibitem[{Zhao et~al.(2025{\natexlab{a}})Zhao, Liu, Huang, Liu, Wang, Liu,
  Tian, Pang, Bell, Grover, and
  Chen}]{zhao2025inpaintingguidedpolicyoptimizationdiffusion}
Zhao, S.; Liu, M.; Huang, J.; Liu, M.; Wang, C.; Liu, B.; Tian, Y.; Pang, G.;
  Bell, S.; Grover, A.; and Chen, F. 2025{\natexlab{a}}.
\newblock Inpainting-Guided Policy Optimization for Diffusion Large Language
  Models.
\newblock arXiv:2509.10396.

\bibitem[{Zhao et~al.(2025{\natexlab{b}})Zhao, Jin, Wang, and
  You}]{zhao2025realtimevideogenerationpyramid}
Zhao, X.; Jin, X.; Wang, K.; and You, Y. 2025{\natexlab{b}}.
\newblock Real-Time Video Generation with Pyramid Attention Broadcast.
\newblock arXiv:2408.12588.

\bibitem[{Zheng et~al.(2026{\natexlab{a}})Zheng, Chen, Ye, Wang, Zhang, Jiang,
  Su, Ermon, Zhu, and Liu}]{zheng2026diffusionnftonlinediffusionreinforcement}
Zheng, K.; Chen, H.; Ye, H.; Wang, H.; Zhang, Q.; Jiang, K.; Su, H.; Ermon, S.;
  Zhu, J.; and Liu, M.-Y. 2026{\natexlab{a}}.
\newblock DiffusionNFT: Online Diffusion Reinforcement with Forward Process.
\newblock arXiv:2509.16117.

\bibitem[{Zheng et~al.(2026{\natexlab{b}})Zheng, Kong, Wu, Jiang, Ma, He, Lin,
  Gong, Zhong, Bo, Chen, and
  Yang}]{zheng2026manifoldawareexplorationreinforcementlearning}
Zheng, M.; Kong, W.; Wu, Y.; Jiang, D.; Ma, Y.; He, X.; Lin, B.; Gong, K.;
  Zhong, Z.; Bo, L.; Chen, Q.; and Yang, H. 2026{\natexlab{b}}.
\newblock Manifold-Aware Exploration for Reinforcement Learning in Video
  Generation.
\newblock arXiv:2603.21872.

\bibitem[{Zhong et~al.(2026)Zhong, Wang, Ding, Feng, Bai, Xiang, Sun, and
  Xu}]{zhong2026stabilizingreinforcementlearningdiffusion}
Zhong, J.; Wang, K.; Ding, D.; Feng, Z.; Bai, H.; Xiang, Y.; Sun, J.; and Xu,
  Q. 2026.
\newblock Stabilizing Reinforcement Learning for Diffusion Language Models.
\newblock arXiv:2603.06743.

\bibitem[{Zhu et~al.(2024)Zhu, Zhao, He, Zhong, Zhang, Guo, Chen, and
  Zhang}]{zhu2024diffusionmodelsreinforcementlearning}
Zhu, Z.; Zhao, H.; He, H.; Zhong, Y.; Zhang, S.; Guo, H.; Chen, T.; and Zhang,
  W. 2024.
\newblock Diffusion Models for Reinforcement Learning: A Survey.
\newblock arXiv:2311.01223.

\end{thebibliography}
\newpage
\appendix

\section{Experiment Figures}
\label{sec:appendix-exp-figs}

This appendix collects the training/evaluation reward curves referenced from the Experiments section, so that the main text can focus on the quantitative results and analysis.

\begin{figure}[h]
    \centering
    \includegraphics[width=0.75\linewidth]{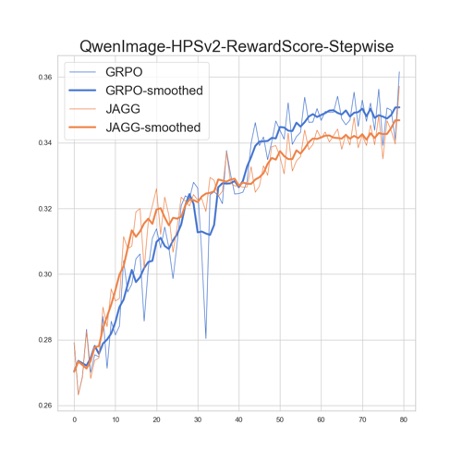}
    \caption{QwenImage HPSv2.1-only smoke test. HPSv2.1 reward score vs.\ training steps for Baseline (blue) and JAGG (orange).}
    \label{fig:qwen_hps_smoke}
\end{figure}

\begin{figure}[h]
    \centering
    \includegraphics[width=0.75\linewidth]{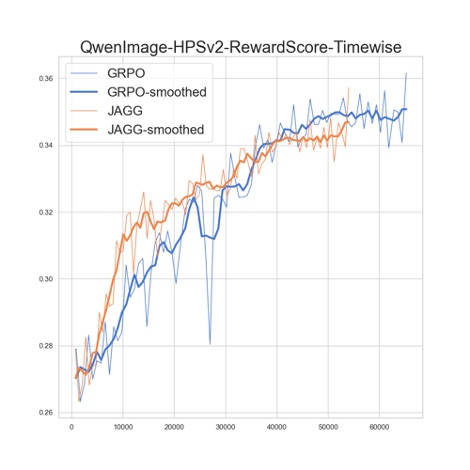}
    \caption{QwenImage training with mixed HPSv2.1+CLIPScore rewards. Average reward score vs.\ training steps for Baseline (blue), JAGG (orange), and the ablation variant when available.}
    \label{fig:qwen_mixed_stepwise}
\end{figure}

\begin{table*}[h]
\centering
\small
\caption{QwenImage trained on \textbf{HPSv2.1 + CLIPScore (1:1)}. All three metrics reported at each checkpoint. Bold indicates best per row.}
\label{tab:qwen_mixed}
\setlength{\tabcolsep}{4pt}
\begin{tabular}{l ccc ccc ccc}
\hline
 & \multicolumn{3}{c}{\textbf{Baseline}} & \multicolumn{3}{c}{\textbf{JAGG (ours)}} & \multicolumn{3}{c}{\textbf{Ablate}} \\
\cmidrule(lr){2-4}\cmidrule(lr){5-7}\cmidrule(lr){8-10}
\textbf{Step} & HPS & Pick & Clip & HPS & Pick & Clip & HPS & Pick & Clip \\
\hline
0  & .2665 & .2224 & .2694 & .2665 & .2224 & .2694 & .2665 & .2224 & .2694 \\
10 & .3053 & .2316 & .2823 & .3038 & .2313 & .2809 & .3068 & .2314 & .2814 \\
20 & .3144 & .2321 & .2851 & \textbf{.3229} & \textbf{.2351} & .2783 & .3181 & .2332 & .2803 \\
30 & .3208 & .2322 & \textbf{.2863} & \textbf{.3284} & \textbf{.2353} & .2807 & .3154 & .2319 & .2721 \\
40 & .3298 & .2342 & \textbf{.2817} & \textbf{.3335} & \textbf{.2360} & .2802 & .3223 & .2329 & .2772 \\
50 & \textbf{.3353} & \textbf{.2350} & \textbf{.2790} & .3322 & .2335 & .2810 & .3295 & .2326 & .2812 \\
60 & .3281 & .2321 & .2744 & \textbf{.3379} & .2333 & .2754 & .3318 & .2318 & .2763 \\
70 & .3392 & .2329 & \textbf{.2742} & \textbf{.3500} & \textbf{.2347} & .2660 & .3396 & .2292 & .2621 \\
80 & .3241 & .2272 & .2619 & .3448 & .2260 & \textbf{.2729} & \textbf{.3452} & .2308 & .2598 \\
\hline
\end{tabular}
\end{table*}

\begin{figure}[h]
    \centering
    \includegraphics[width=0.75\linewidth]{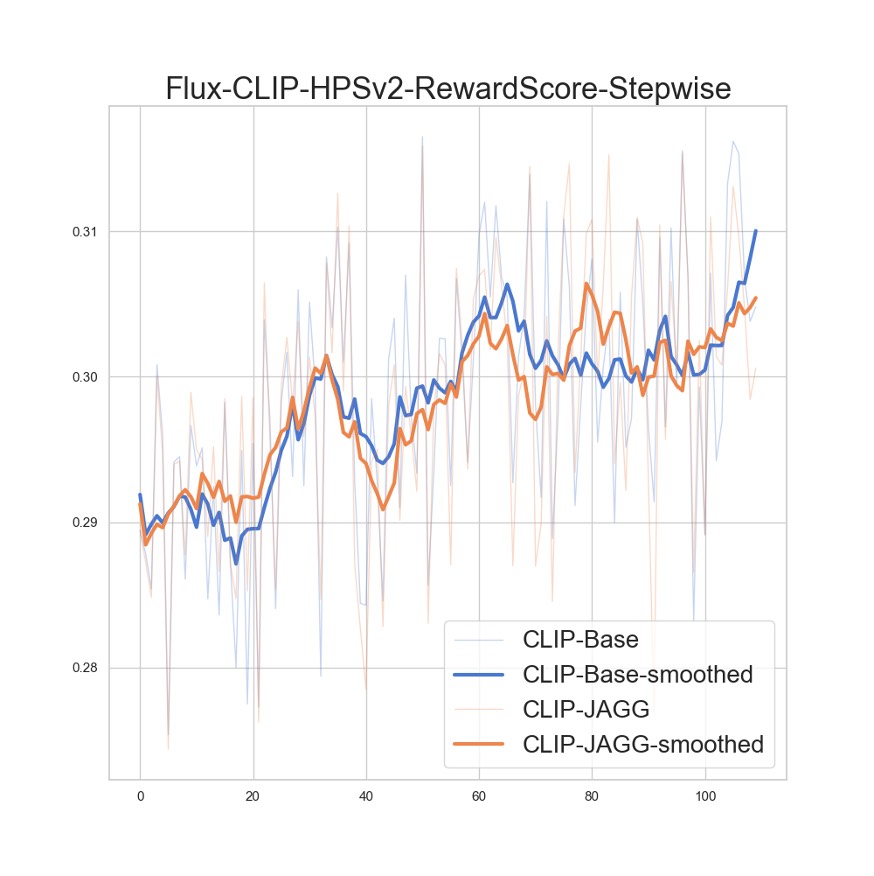}
    \caption{Flux DanceGRPO training with mixed HPSv2.1+CLIPScore rewards. Reward curve for Baseline, JAGG, and the ablation setting where available.}
    \label{fig:flux_mixed_dancegrpo}
\end{figure}

\begin{figure}[h]
    \centering
    \includegraphics[width=0.48\linewidth]{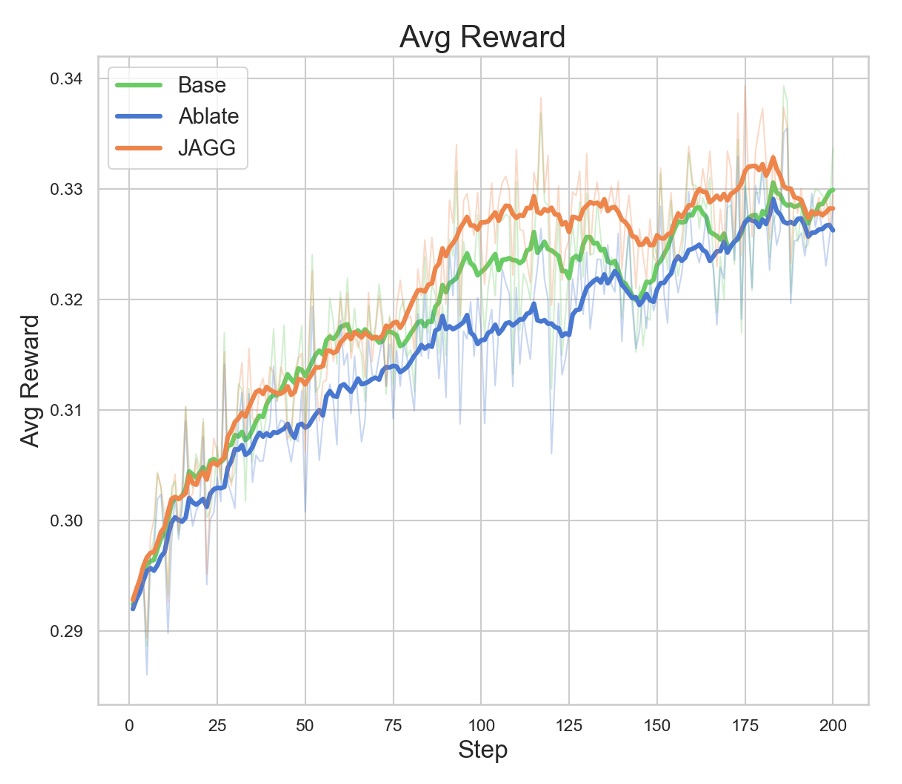}\hfill\includegraphics[width=0.48\linewidth]{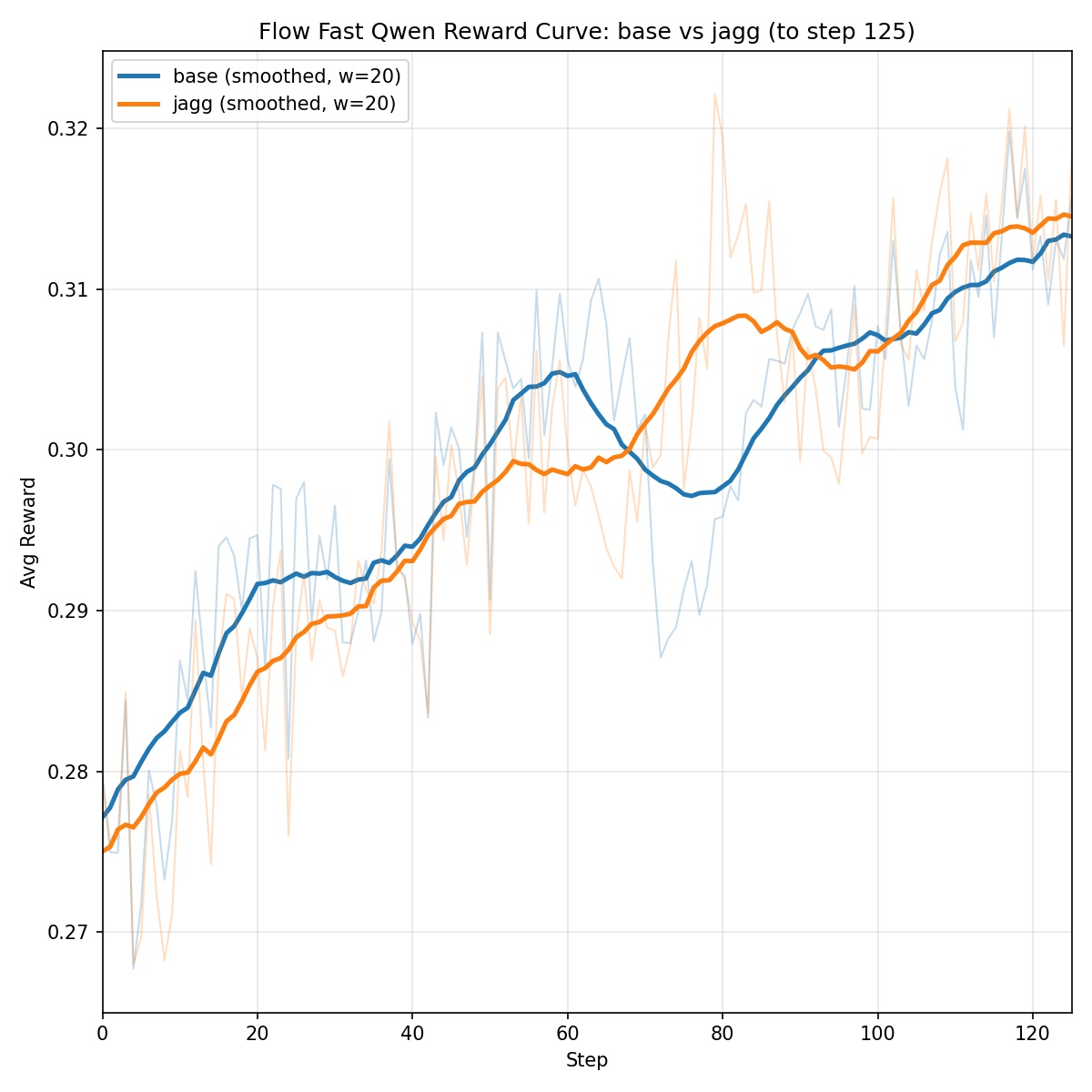}
    \caption{\textbf{Left:} Flux1.dev and \textbf{Right:} QwenImage reward curve trained using Flow-GRPO-Fast algorithm under different settings (w JAGG, w/o JAGG, Ablation).}
    \label{fig:flow_fast_run_stepwise}
\end{figure}

\FloatBarrier

\section{Why Backward-Pass Reduction Yields Outsized Speedup}
\label{app:why-bwd-speedup}

This appendix expands on the two mechanisms briefly referenced in the main-text Efficiency Analysis: (i) the transformer backward pass is intrinsically more expensive than the forward pass in FLOPs, and (ii) each backward pass triggers a collective communication under data parallelism. Together they explain why cutting the number of backward passes from $W$ to $2$ per group yields a wall-clock speedup that is substantially larger than what raw FLOP counts alone would suggest.

\noindent\textbf{Backward-pass compute is inherently heavier than forward.}
Consider the scaled dot-product attention sublayer. The forward pass computes:
\begin{equation}
    S = QK^\top/\sqrt{d_k}, \quad P = \text{softmax}(S), \quad O = PV,
\end{equation}
where $Q,K,V\in\mathbb{R}^{N\times d_k}$. The dominant cost is the two matrix products $QK^\top$ and $PV$, each $O(N^2 d_k)$ FLOPs, giving a total forward attention cost of $2N^2 d_k$ FLOPs (ignoring the cheap softmax). In the backward pass, given upstream gradient $\partial L/\partial O$, one must compute:
\begin{align}
    \frac{\partial L}{\partial V} &= P^\top \frac{\partial L}{\partial O}, \label{eq:bwd_V}\\
    \frac{\partial L}{\partial P} &= \frac{\partial L}{\partial O} V^\top, \label{eq:bwd_P}\\
    \frac{\partial L}{\partial S} &= \text{softmax\_bwd}\!\left(P,\, \frac{\partial L}{\partial P}\right), \label{eq:bwd_S}\\
    \frac{\partial L}{\partial Q} &= \frac{\partial L}{\partial S} K/\sqrt{d_k}, \quad
    \frac{\partial L}{\partial K} = \left(\frac{\partial L}{\partial S}\right)^\top Q/\sqrt{d_k}. \label{eq:bwd_QK}
\end{align}
Equations~\eqref{eq:bwd_V}--\eqref{eq:bwd_QK} involve \emph{four} $O(N^2 d_k)$ matrix products (for $\partial L/\partial V$, $\partial L/\partial P$, $\partial L/\partial Q$, $\partial L/\partial K$) plus the $O(N^2)$ softmax backward---totaling $4N^2 d_k$ dominant FLOPs, exactly $\mathbf{2\times}$ the forward attention cost. Accounting for the weight-gradient terms in the linear projections, the full transformer backward is approximately $2\times$ the forward in total FLOPs. Reducing the number of backward passes therefore yields disproportionately larger compute savings than an equal reduction in forward passes.

\noindent\textbf{All-reduce communication is coupled to every backward pass under data parallelism.}
In distributed training with data parallelism (DP) or FSDP, each backward pass triggers an all-reduce (or reduce-scatter) over the full parameter gradient tensor. This communication is unavoidable regardless of batch size: even a single-sample backward must synchronize gradients across all DP ranks before the optimizer step. In large-scale setups the all-reduce overhead is non-trivial---it grows linearly with parameter count and can constitute a significant fraction of step time at scale. Because JAGG reduces the number of backward passes from $W$ to $2$ per group, it also reduces the number of all-reduce rounds by the same factor, directly cutting communication overhead.

Together, these two effects account for the measured end-to-end speedups reported in the Efficiency Analysis sections (DanceGRPO and Flow-GRPO-Fast, Figure~\ref{fig:efficiency_flowfast}).

\section{JAGG Implementation Pseudocode}\begin{algorithm}[H]
\caption{JAGG: Joint-Graph Variant (memory-standard)}
\label{alg:jagg-joint}
\begin{algorithmic}[1]
\REQUIRE Group latents $\{z_j\}_{j=0}^{W-1}$, timesteps $\{t_j\}$, advantage $A$, old log-probs $\{\ell_j^{\text{old}}\}$
\STATE // Phase 1: forward passes
\FOR{$j = 0$ \TO $W-1$}
  \IF{$j \in \{0, W-1\}$}
    \STATE $\hat{v}_j \leftarrow \mathrm{DiT}(z_j, t_j)$ \hfill // with grad
  \ELSE
    \STATE $\hat{v}_j \leftarrow \mathrm{DiT}(z_j, t_j)$ \hfill // no\_grad
  \ENDIF
\ENDFOR
\STATE // Phase 2: upstream gradients via detached proxy
\STATE $s_{\text{base}} \leftarrow 0$; $\; s_{\text{corr}} \leftarrow 0$
\FOR{$j = 0$ \TO $W-1$}
  \STATE $\tilde{v}_j \leftarrow \hat{v}_j.\mathrm{detach}().\mathrm{requires\_grad\_}()$
  \STATE Compute $L_j$ (PPO loss) using $\tilde{v}_j$, $z_j$, $\ell_j^{\text{old}}$, $A$
  \STATE $s_j \leftarrow \partial L_j / \partial \tilde{v}_j$
  \STATE $\alpha_j \leftarrow (t_0 - t_j)/(t_0 - t_{W-1})$
  \STATE $s_{\text{base}} \mathrel{+}= (1-\alpha_j)\,s_j$; $\; s_{\text{corr}} \mathrel{+}= \alpha_j\,s_j$
\ENDFOR
\STATE // Phase 3: joint backward (single FSDP all-gather)
\STATE $\texttt{torch.autograd.backward}([\hat{v}_0,\, \hat{v}_{W-1}],\,[s_{\text{base}},\, s_{\text{corr}}])$
\end{algorithmic}
\end{algorithm}

\begin{algorithm}[H]
\caption{JAGG: Sequential-Graph Variant (memory-efficient, for large models)}
\label{alg:jagg-seq}
\begin{algorithmic}[1]
\REQUIRE Same as Algorithm~\ref{alg:jagg-joint}
\STATE // Phase 1: no-grad forward for middle steps; accumulate partial signals
\STATE $s_{\text{base}} \leftarrow 0$; $\; s_{\text{corr}} \leftarrow 0$
\FOR{$j = 1$ \TO $W-2$}
  \STATE $\hat{v}_j \leftarrow \mathrm{DiT}(z_j, t_j)$ \hfill // no\_grad
  \STATE $\tilde{v}_j \leftarrow \hat{v}_j.\mathrm{detach}().\mathrm{requires\_grad\_}()$; compute $s_j$
  \STATE $\alpha_j \leftarrow (t_0 - t_j)/(t_0 - t_{W-1})$
  \STATE $s_{\text{base}} \mathrel{+}= (1-\alpha_j)\,s_j$; $\; s_{\text{corr}} \mathrel{+}= \alpha_j\,s_j$
\ENDFOR
\STATE // Phase 2a: endpoint $j{=}0$ ($\alpha_0{=}0$, contributes entirely to $s_{\text{base}}$)
\STATE $\hat{v}_0 \leftarrow \mathrm{DiT}(z_0, t_0)$ \hfill // with grad; graph built
\STATE $\tilde{v}_0 \leftarrow \hat{v}_0.\mathrm{detach}().\mathrm{requires\_grad\_}()$; compute $s_0$
\STATE $s_{\text{base}} \mathrel{+}= s_0$
\STATE $\hat{v}_0.\mathrm{backward}(s_{\text{base}})$ \hfill // frees graph for step 0
\STATE // Phase 2b: endpoint $j{=}W{-}1$ ($\alpha_{W-1}{=}1$, contributes entirely to $s_{\text{corr}}$)
\STATE $\hat{v}_{W-1} \leftarrow \mathrm{DiT}(z_{W-1}, t_{W-1})$ \hfill // with grad; graph built
\STATE $\tilde{v}_{W-1} \leftarrow \hat{v}_{W-1}.\mathrm{detach}().\mathrm{requires\_grad\_}()$; compute $s_{W-1}$
\STATE $s_{\text{corr}} \mathrel{+}= s_{W-1}$
\STATE $\hat{v}_{W-1}.\mathrm{backward}(s_{\text{corr}})$ \hfill // frees graph for step $W{-}1$
\end{algorithmic}
\end{algorithm}

\section{Experimental Setup and Hyperparameters}
\label{sec:appendix-hparams}

This section reports the full hardware and hyperparameter configuration used in our experiments. All training scripts share the same family of arguments built on top of \texttt{torchrun} with FSDP-wrapped DiT backbones; the only systematic differences across runs lie in (i) the underlying base model, (ii) the JAGG-specific arguments (\texttt{sum\_sample}, \texttt{jagg\_frac}), and (iii) the resolution / temporal extent imposed by the generation modality. Inference (rollout) and training share the same code path during on-policy GRPO; we therefore describe their hyperparameters jointly.

\noindent\textbf{Hardware.}
All experiments are conducted on GPUs each with $270$\,GB of memory and high-speed interconnect. Image-generation experiments (Flux, QwenImage) use a single node of $8$ GPUs. Inference (the on-policy rollout that produces the candidate group of $G=12$ samples per prompt) is uniformly carried out on $8$ GPUs with FSDP sharding of the DiT backbone, ensuring a comparable throughput baseline across all reported wall-clock numbers.

\noindent\textbf{Common GRPO and optimization settings.}
We use AdamW with learning rate $1\times 10^{-5}$, weight decay $1\times 10^{-4}$, no warmup, and bf16 mixed precision throughout. Each step uses $G=12$ generations per prompt with intra-group advantage normalization (\texttt{use\_group}). The PPO-style surrogate uses a clip range of $0.1$ and an advantage clamp of $5.0$, with gradient clipping at $1\times 10^{-3}$ to stabilize training under the larger effective batch size implied by \texttt{gradient\_accumulation\_steps}. All runs use seed 42, sampler seed $12627$, shared initial noise (\texttt{init\_same\_noise}/\texttt{use\_same\_noise}) within each group, and CFG disabled (\texttt{cfg=0.0}).

\noindent\textbf{Sampling and rollout.}
We use a unified denoising schedule of \texttt{sampling\_steps}$=21$ across all models. The choice of $21$ rather than $20$ ensures that $(\texttt{sampling\_steps}-1)=20$ is divisible by the JAGG group size $W \in \{4, 5\}$, which is required for clean grouping. The flow-matching shift parameter is $3$ for image models, with stochastic-DDIM noise scale $\eta=0.3$ (Flux JAGG additionally uses $\eta=0.9$ for higher exploration). The image runs render at $720\times 720$ ($t=1$).

\noindent\textbf{Per-experiment configuration.}
Table~\ref{tab:hparams} reports the per-model settings for QwenImage and Flux. Within each model the standard baseline and the JAGG variant share the same hyperparameters; the only differences are the JAGG-specific arguments (\texttt{sum\_sample}, \texttt{jagg\_frac}, JAGG variant), which are inactive in the baseline and listed at the bottom of the table. Per-GPU batch size is $8$ throughout; the effective on-policy batch size per optimizer step is \texttt{train\_batch\_size}$\times$\texttt{nproc}$\times$\texttt{gradient\_accumulation\_steps} divided by \texttt{sp\_size}.

\begin{table*}[h]
\centering
\small
\caption{Per-model hyperparameters. Within each model column, the standard baseline and the JAGG variant use identical settings except for the JAGG-specific rows at the bottom (inactive for the baseline). All runs use $G=12$ generations per prompt, lr $1\times 10^{-5}$, weight decay $1\times 10^{-4}$, AdamW, and bf16. Training: $8$ GPUs. Inference: uniformly $8$ GPUs with FSDP.}
\label{tab:hparams}
\begin{tabular}{lcc}
\toprule
Hyperparameter & QwenImage & Flux \\
\midrule
Modality / resolution & T2I, $720^2$ & T2I, $720^2$ \\
Training GPUs & $8$ & $8$ \\
\texttt{train\_batch\_size} (per GPU) & $8$ & $8$ \\
\texttt{gradient\_accumulation\_steps} & $4$ & $12$ \\
\texttt{sampling\_steps} & $21$ & $21$ \\
\texttt{shift} / $\eta$ & $3$ / $0.3$ & $3$ / $0.3$ \\
\texttt{timestep\_fraction} & $0.6$ & $1.0$ \\
\texttt{max\_grad\_norm} & $1.0$ & $1\times 10^{-3}$ \\
\texttt{clip\_range} / \texttt{adv\_clip\_max} & $1\times 10^{-1}$ / $5.0$ & $1\times 10^{-1}$ / $5.0$ \\
\texttt{ignore\_last} & yes & yes \\
\texttt{sp\_size} & $1$ & $1$ \\
\texttt{selective\_checkpointing} & $0.0$ & $0.0$ \\
\texttt{use\_cpu\_offload} & no & no \\
\midrule
\multicolumn{3}{l}{\emph{JAGG-only (inactive in the standard baseline)}} \\
\texttt{sum\_sample} ($W$) & $4$ & $4$ \\
\texttt{jagg\_frac} & $0.6$ & $0.6$ \\
JAGG variant & joint-graph & joint-graph \\
\midrule
\multicolumn{3}{l}{\emph{Flow-GRPO-Fast-specific settings}} \\
\texttt{sde\_window} (step range) & $4$--$12$ & $0$--$12$ \\
KL coefficient & $0.001$ & $0.1$ \\
\texttt{gradient\_accumulation\_steps} & $96$ & $12$ \\
\texttt{jagg\_frac} (for stability) & $0.4$ & $0.4$ \\
\bottomrule
\end{tabular}
\end{table*}

\noindent\textbf{Reward configuration.}
T2I runs (Flux, QwenImage) optimize a weighted combination of HPSv2~\cite{wu2023humanpreferencescorev2} (\texttt{use\_hpsv2}) and CLIPScore~\cite{hessel2022clipscorereferencefreeevaluationmetric} (\texttt{use\_clipscore}, with \texttt{clip-vit-large-patch14}).

\noindent\textbf{Memory and parallelism knobs.}
All runs enable \texttt{gradient\_checkpointing} and FSDP. Image runs use $\texttt{sp\_size}=1$ (no sequence-parallel) and \texttt{selective\_checkpointing}$=0.0$. Image runs use the \emph{joint-graph} variant of JAGG (Algorithm~\ref{alg:jagg-joint}) for fewer FSDP all-gathers.

\noindent\textbf{Timestep filtering.}
The argument \texttt{timestep\_fraction} controls the fraction of the trajectory eligible for gradient updates. We use $1.0$ (full trajectory) for Flux and $0.6$ for QwenImage, which we found to slightly stabilize early training under the larger DiT. The flag \texttt{ignore\_last} drops the terminal (clean-sample) timestep from the loss for image runs, consistent with prior diffusion-RL practice. The JAGG-specific routing parameter \texttt{jagg\_frac} sets the fraction of \emph{groups} (ordered from high-noise to low-noise) that use 2-pass JAGG backward; the remaining late, near-clean groups use exact $W$-pass backward, as motivated in the ``\nameref{sec:empirical-quality}'' section.

\section{Empirical Gradient Quality and Choice of \texttt{jagg\_frac}}
\label{sec:empirical-quality}

\begin{table}[htbp]
\centering
\small
\caption{Per-group gradient approximation quality of JAGG ($20$ denoising steps, $W{=}4$, $5$ groups; group~0 = highest noise, group~4 = lowest noise). Bold entries (satisfying $\cos\geq 0.5$) are routed to JAGG; the rest use exact per-step backward. $\rho$ is reported as an additional quality indicator; its dramatic inflation in late groups corroborates the routing decision. This yields \texttt{jagg\_frac}$=0.6$ for T2I.}
\label{tab:jagg_quality}
\begin{tabular}{clcc}
\hline
\textbf{Model} & \textbf{Group} & $\cos(\hat{g}_\text{JAGG}, g_\text{exact})$ & $\rho$ \\
\hline
\multirow{5}{*}{Flux}
 & 0 (high noise) & $\mathbf{0.85 \pm 0.06}$ & $\mathbf{0.96 \pm 0.08}$ \\
 & 1              & $\mathbf{0.77 \pm 0.18}$ & $\mathbf{1.03 \pm 0.20}$ \\
 & 2              & $\mathbf{0.74 \pm 0.04}$ & $\mathbf{1.08 \pm 0.09}$ \\
 & 3              & $\mathbf{0.63 \pm 0.13}$ & $1.22 \pm 0.21$ \\
 & 4 (low noise)  & $0.41 \pm 0.14$ & $1.56 \pm 0.24$ \\
\hline
\multirow{5}{*}{QwenImage}
 & 0 (high noise) & $\mathbf{0.79 \pm 0.09}$ & $\mathbf{1.15 \pm 0.14}$ \\
 & 1              & $\mathbf{0.64 \pm 0.19}$ & $\mathbf{1.05 \pm 0.25}$ \\
 & 2              & $\mathbf{0.58 \pm 0.26}$ & $\mathbf{1.19 \pm 0.31}$ \\
 & 3              & $0.45 \pm 0.12$ & $1.79 \pm 0.30$ \\
 & 4 (low noise)  & $0.22 \pm 0.28$ & $4.35 \pm 0.95$ \\
\hline
\end{tabular}
\end{table}
Theorem~\ref{thm:linear_interp} guarantees exact gradient recovery under the linearity assumption. To verify that this transfers to real diffusion backbones---and in particular to confirm that \emph{JAGG provides accurate gradient estimates in the early, high-noise denoising steps}---we directly measure the approximation quality on rollouts from Flux and QwenImage. For each rollout we use $20$ denoising steps split into $5$ groups of $W{=}4$. For every group we compute (i) the JAGG-aggregated gradient $\hat{g}_\text{JAGG}$ via the two-endpoint backward and (ii) the exact reference gradient $g_\text{exact}$ from full per-step backward through all $W$ velocity predictions. We report the cosine similarity $\cos(\hat{g}_\text{JAGG}, g_\text{exact})$ and the norm ratio $\rho = \|\hat{g}_\text{JAGG}\|/\|g_\text{exact}\|$, averaged over $\geq 20$ rollouts per group.

As shown in Table~\ref{tab:jagg_quality}, in the high-noise groups (0--1) JAGG produces gradients very close to the exact reference on both backbones---cosine similarities are $0.85$/$0.77$ (Flux) and $0.79$/$0.64$ (QwenImage), with norm ratio $\rho$ within $\pm 50\%$ of $1$. This confirms the central premise of JAGG: in the early denoising steps the velocity field is sufficiently close to linear that two-endpoint Jacobian aggregation recovers the true descent direction faithfully. The approximation degrades monotonically as noise decreases, with cosine similarity dropping below $0.7$ at group~3 for both T2I models, while $\rho$ inflates dramatically in later groups---reaching $4.35\times$ at group~4 for QwenImage---precisely the signal that motivates the joint routing rule to exclude these groups.

This measurement directly determines \texttt{jagg\_frac}: we route a group to JAGG when its measured $\cos(\hat{g}_\text{JAGG}, g_\text{exact}) \geq 0.5$. A cosine similarity above $0.5$ ensures that the approximated gradient has a positive projection onto the true descent direction, and optimization theory~\cite{bottou2018optimizationmethodslargescalemachine} establishes that such a sufficient-descent-direction condition is sufficient for convergence of gradient-based methods; our training uses Adam/AdamW which, as a preconditioned SGD-with-momentum optimizer, inherits the same robustness to biased gradient directions. Applying this rule: for Flux, groups 0--3 all exceed $0.5$ in mean cosine, though group~3 carries high variance ($\pm 0.13$); we conservatively set \texttt{jagg\_frac}$=0.6$ (groups 0--2) to avoid the borderline group~3. For QwenImage, groups 0--2 qualify, also yielding \texttt{jagg\_frac}$=0.6$. Crucially, the $\rho$ column provides a corroborating signal: its dramatic inflation to $4.35\times$ in the late QwenImage groups confirms that the approximation breaks down there in both direction and scale, validating the routing decision. This data-driven rule exactly reproduces the values used in our main experiments and provides a model-agnostic recipe for any new backbone.

\FloatBarrier
\section{Qualitative Samples}
\label{sec:appendix-samples}

This section provides side-by-side qualitative comparisons between the standard GRPO baseline and our JAGG variant on the HPSv2.1+CLIPScore reward. We render each method's output for the same prompt (i.e., the same trajectory ID and shared initial noise) at a representative training checkpoint. For QwenImage we use \texttt{checkpoint-71}; for Flux we use \texttt{checkpoint-81}. The QwenImage block additionally shows a JAGG ablation (\texttt{JAGG (ablate)}) that drops the $t$-weighted middle-step contribution, which underperforms full JAGG and motivates the use of the full Jacobian aggregation. The 20 prompts per model were sampled to cover a diverse set of styles (anime, fantasy, 3D / Pixar, surreal painting, photorealism, animal portrait, character action, complex scene composition).

\subsection{QwenImage (Step 71)}

\qwenSampleRow{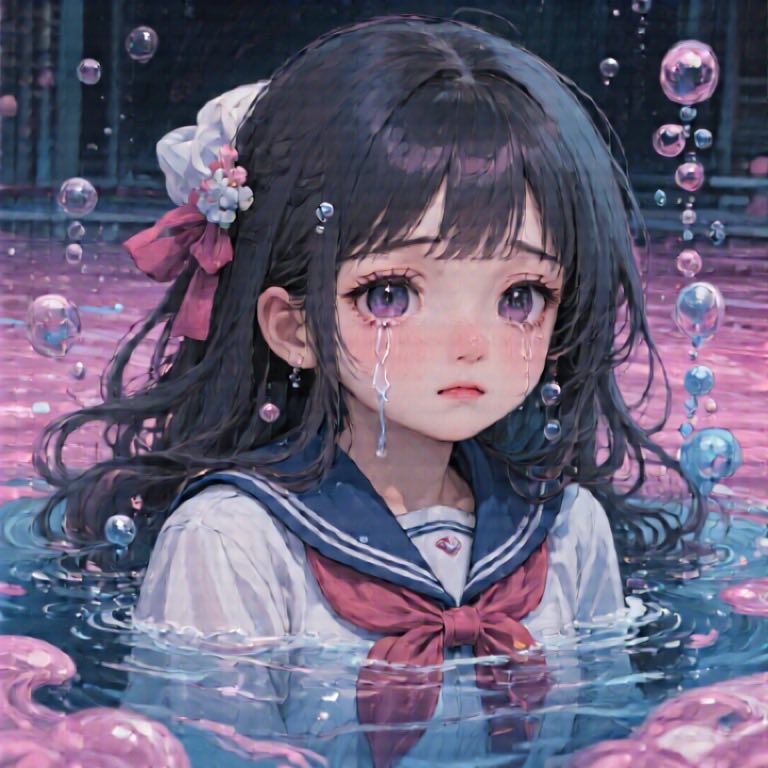}{A cute anime schoolgirl with a sad face submerged in dark pink and blue water, portrayed in an oil painting style.}
\qwenSampleRow{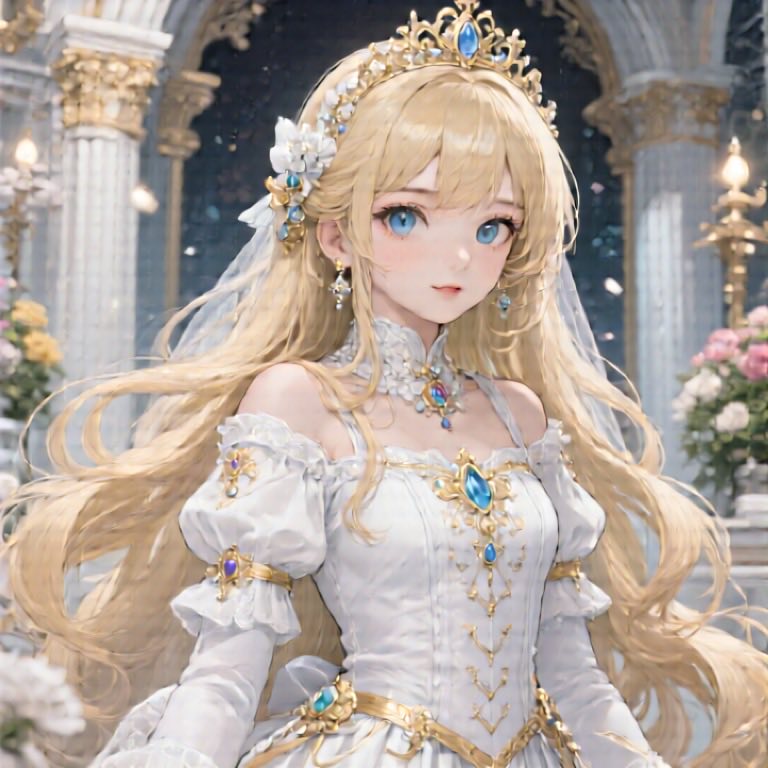}{Portrait of an anime princess in white and golden clothes.}
\qwenSampleRow{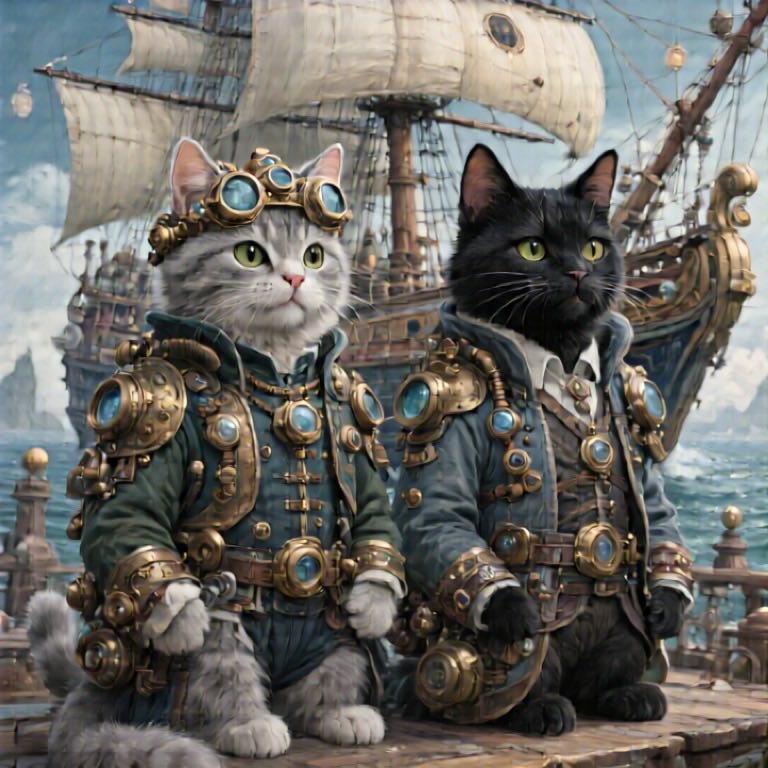}{Two cats, one grey and one black, are wearing steampunk attire and standing in front of a ship in a heavily detailed painting.}
\qwenSampleRow{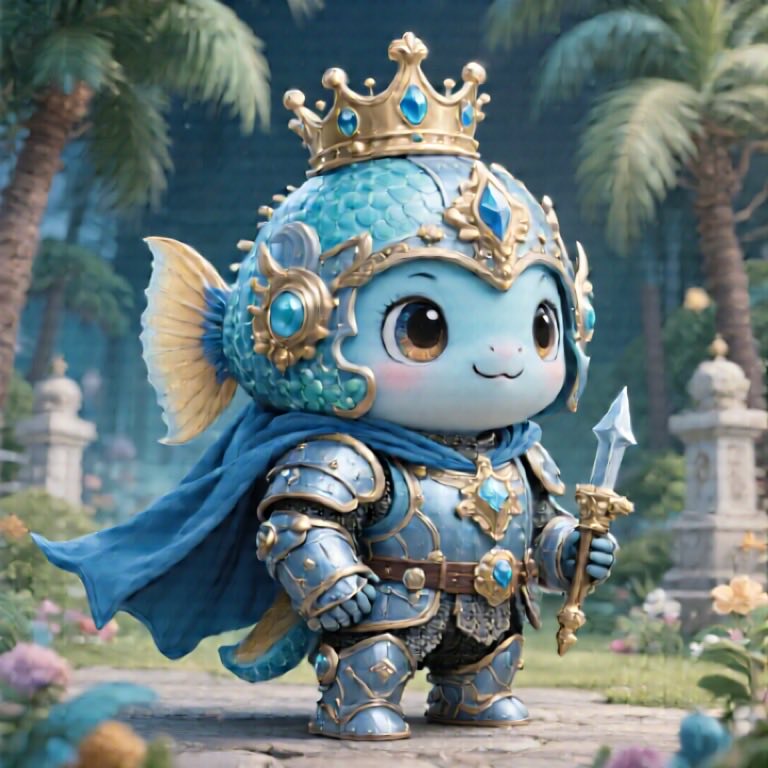}{A cute little anthropomorphic Tropical fish knight wearing a cape and a crown in short, pale blue armor.}
\qwenSampleRow{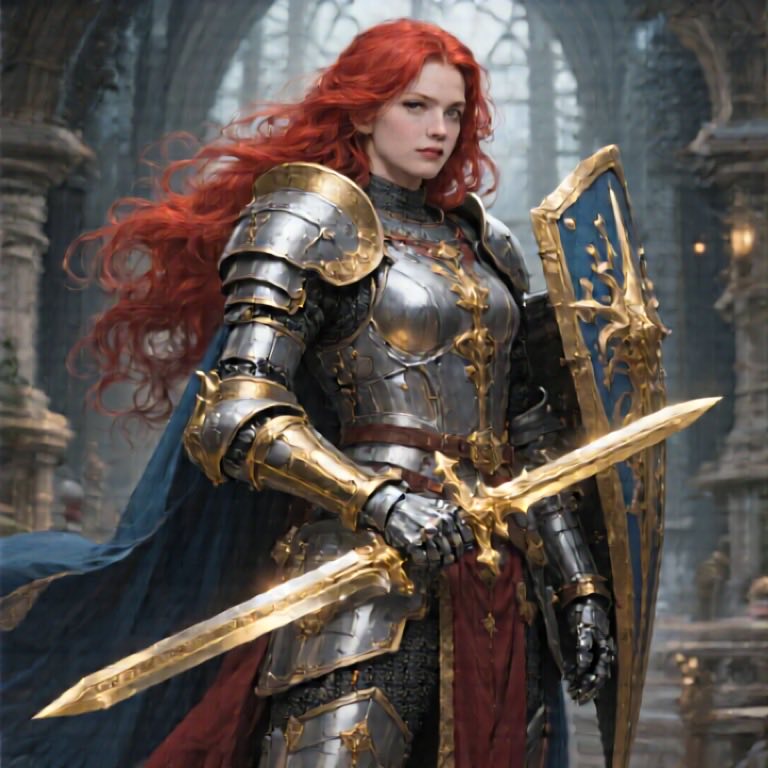}{A red-haired female knight with a golden prosthetic arm wields a long golden blade.}

\qwenSampleRow{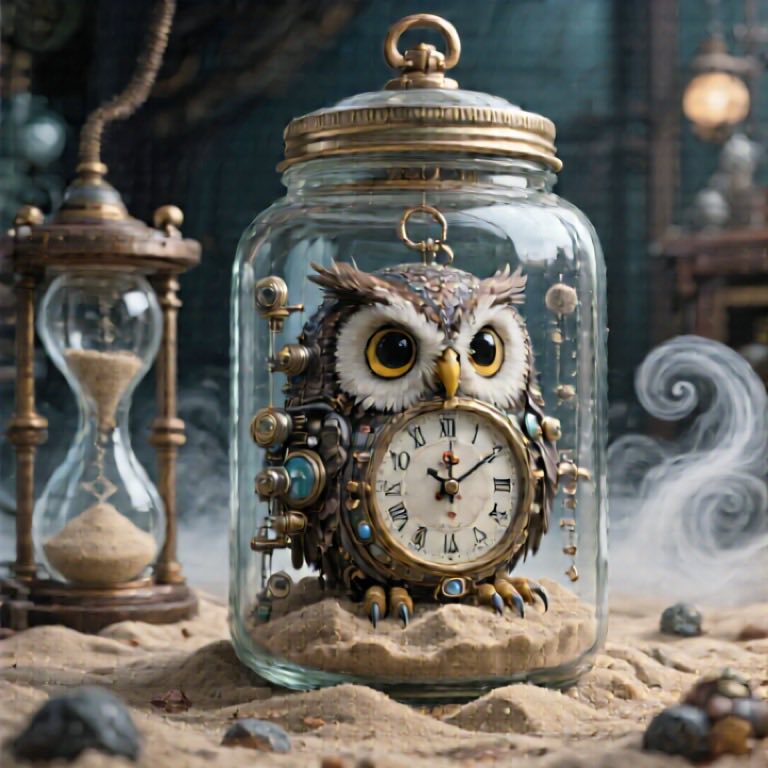}{A steampunk pocketwatch owl is trapped inside a glass jar buried in sand, surrounded by an hourglass and swirling mist.}
\qwenSampleRow{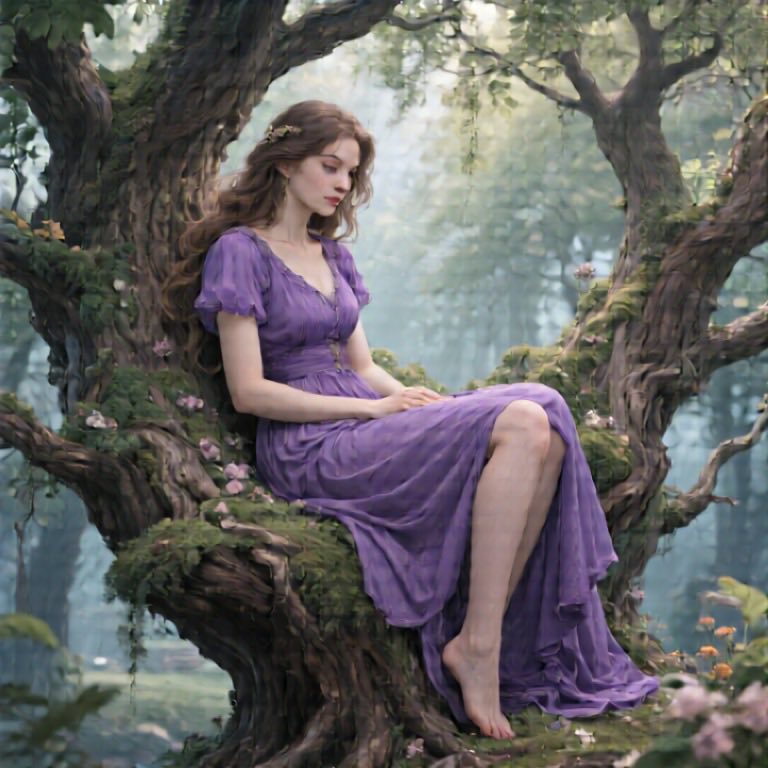}{A lady in a purple dress sitting in a tree --- concept art.}
\qwenSampleRow{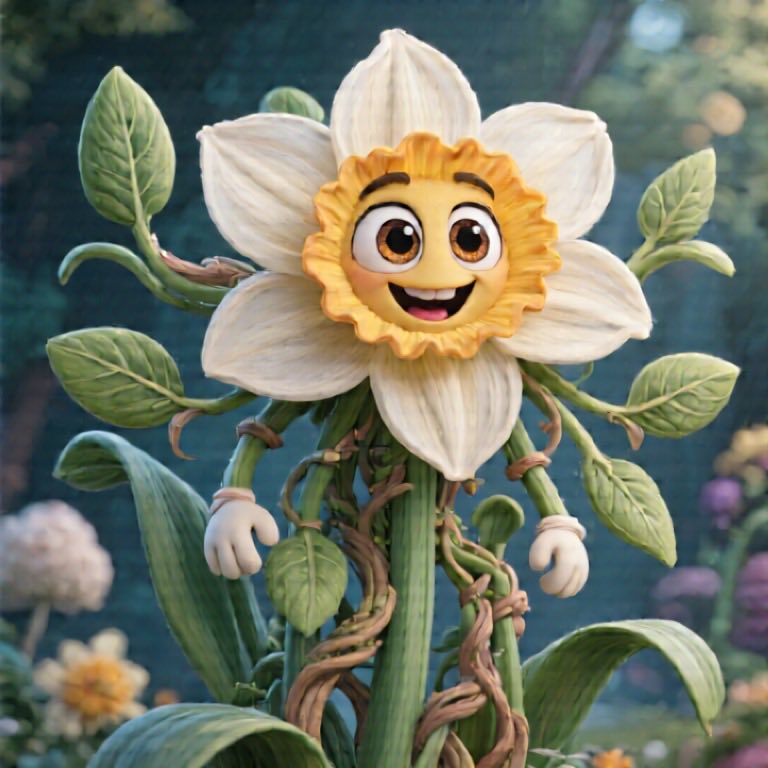}{A happy daffodil with big eyes, multiple leaf arms and vine legs, rendered in 3D Pixar style.}
\qwenSampleRow{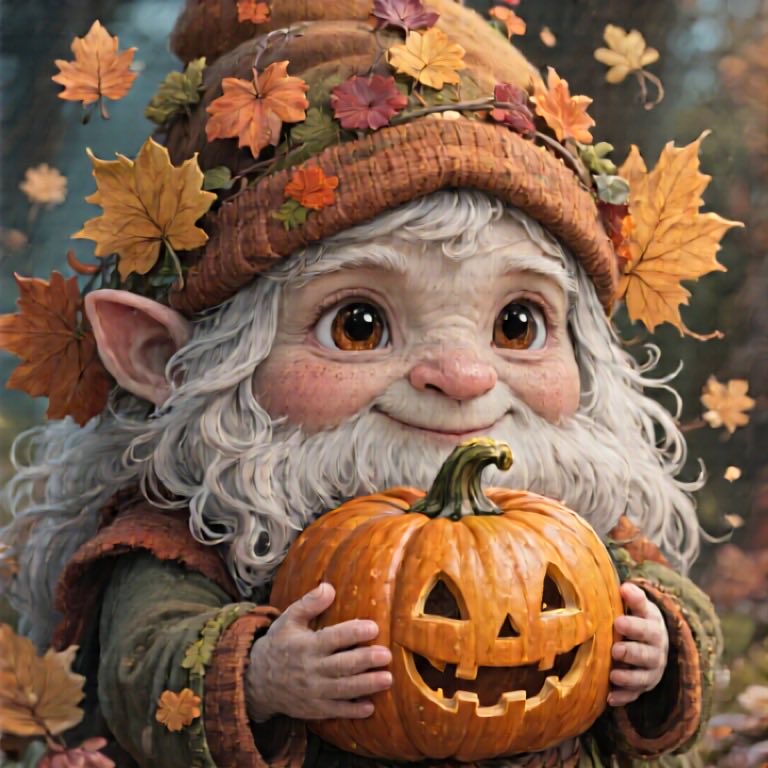}{A hand-drawn cute gnome holding a pumpkin in autumn disguise.}
\qwenSampleRow{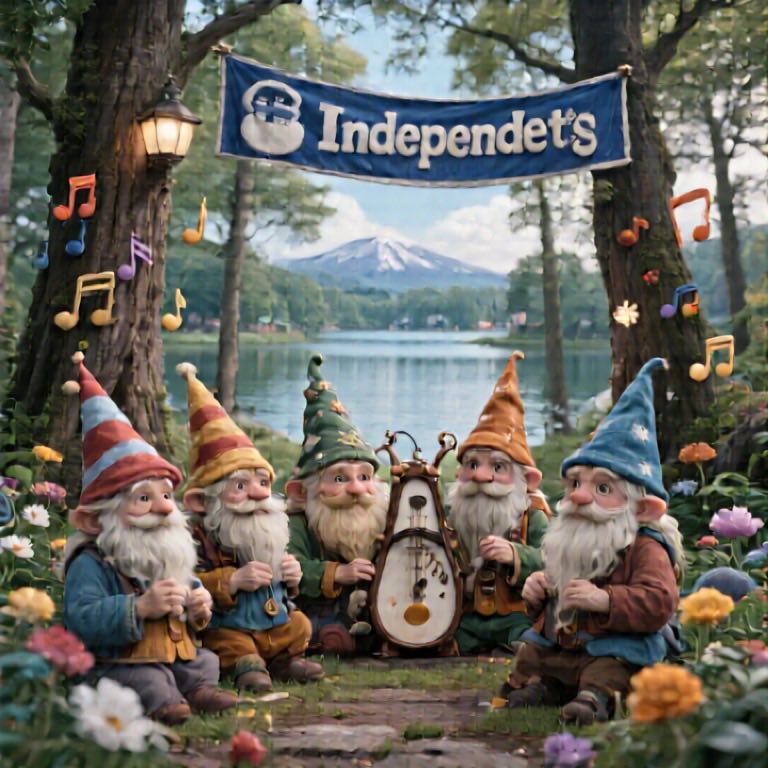}{Gnomes are playing music during Independence Day festivities in a forest near Lake George.}

\qwenSampleRow{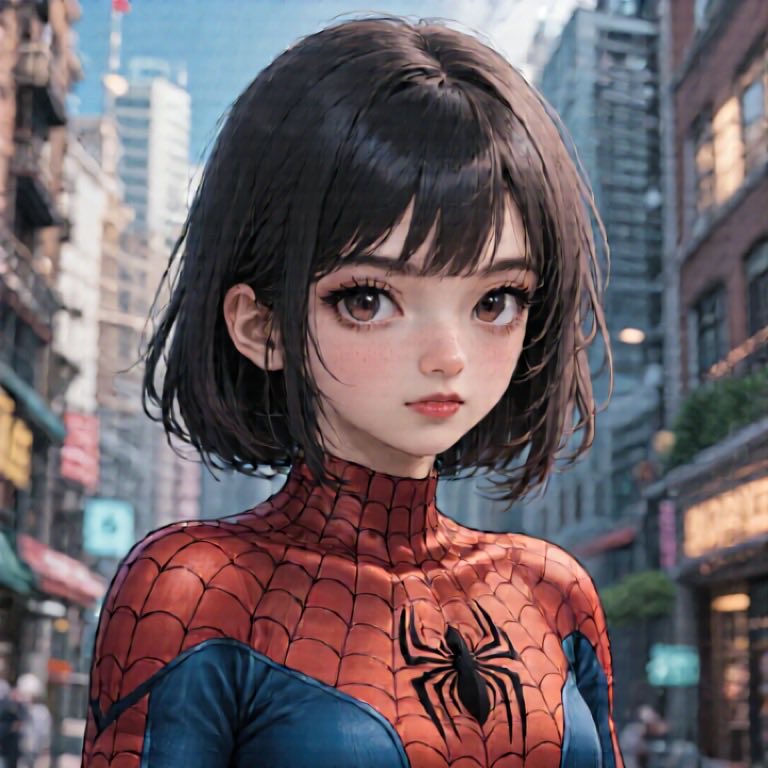}{An anime Spider-Man girl.}
\qwenSampleRow{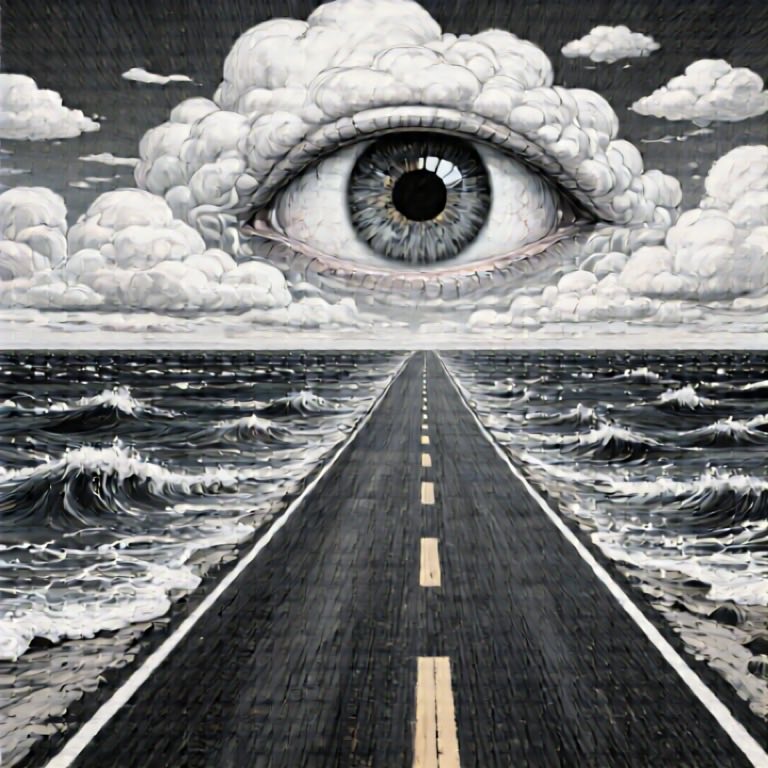}{A black and white drawing of a road splitting the ocean leading to a giant eyeball looking at clouds in the distance.}
\qwenSampleRow{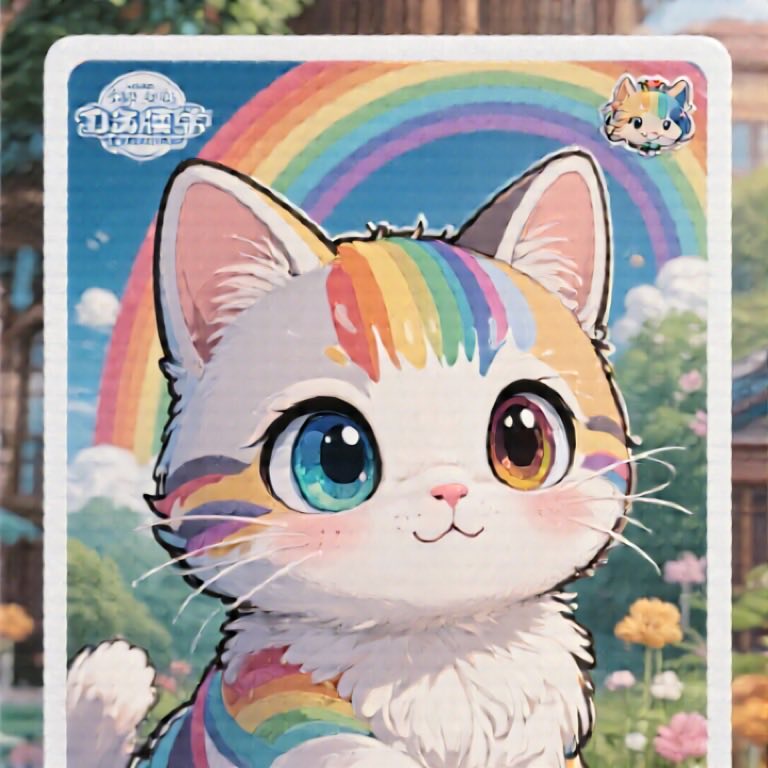}{A cute rainbow kitten with different colored eyes in the chibi-style of Studio Ghibli, featured on a postcard.}
\qwenSampleRow{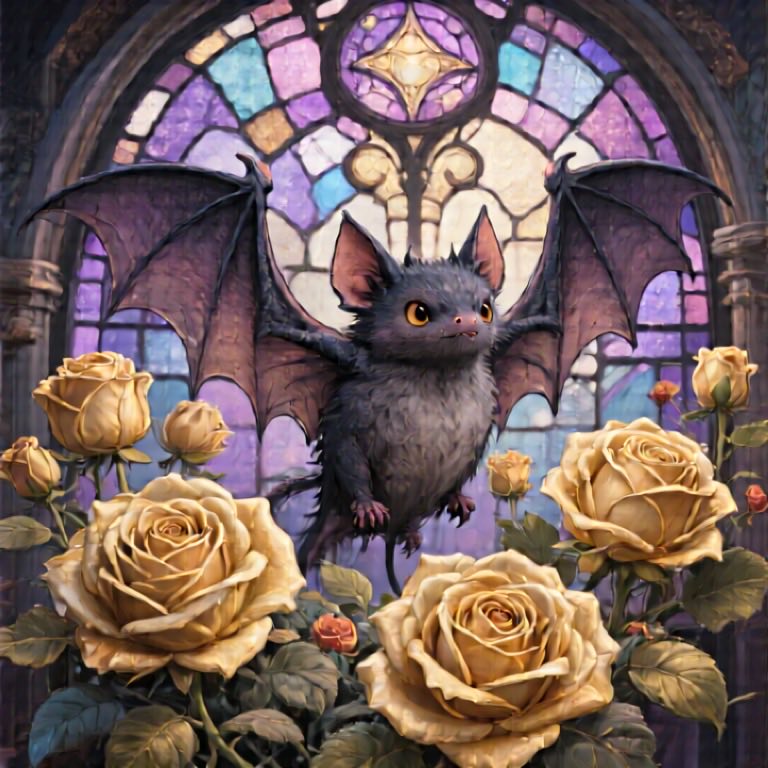}{A detailed soft painting of a bat with golden rose flowers and amethyst stained glass in the background.}
\qwenSampleRow{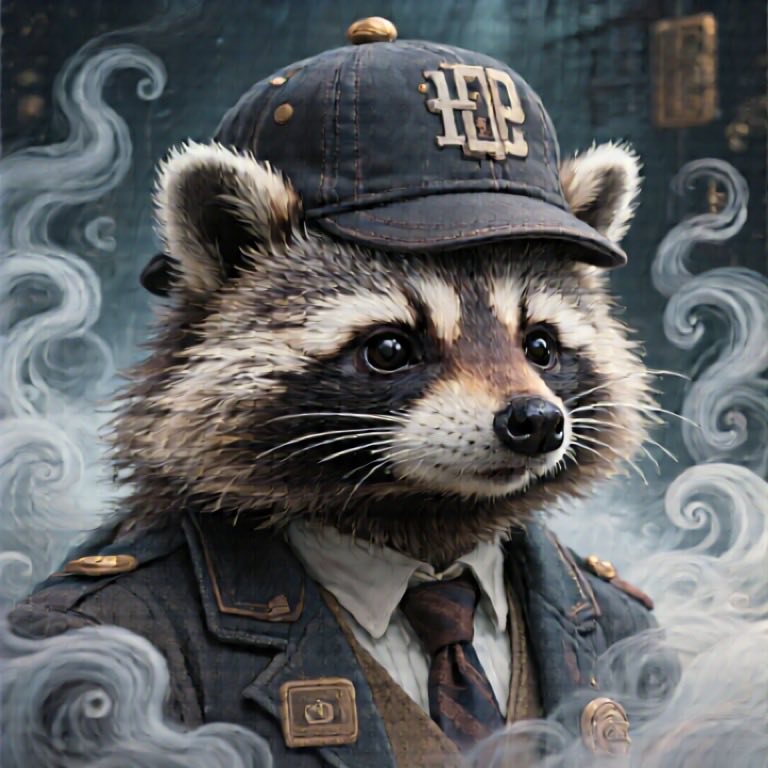}{A raccoon wearing a Peaky Blinders hat, surrounded by swirling mist and rendered with fine detail.}

\qwenSampleRow{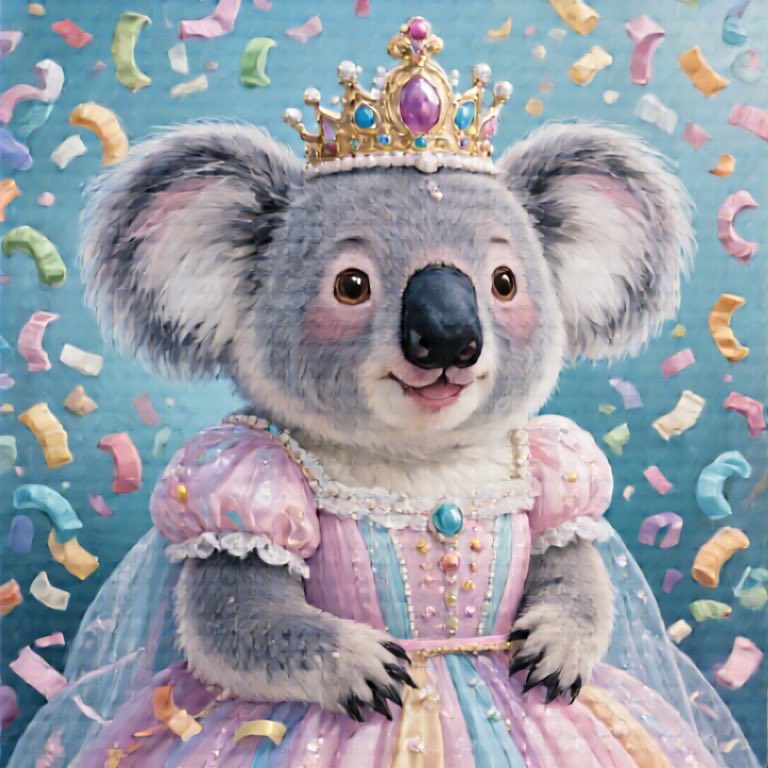}{A painting of a koala wearing a princess dress and crown, with a confetti background.}
\qwenSampleRow{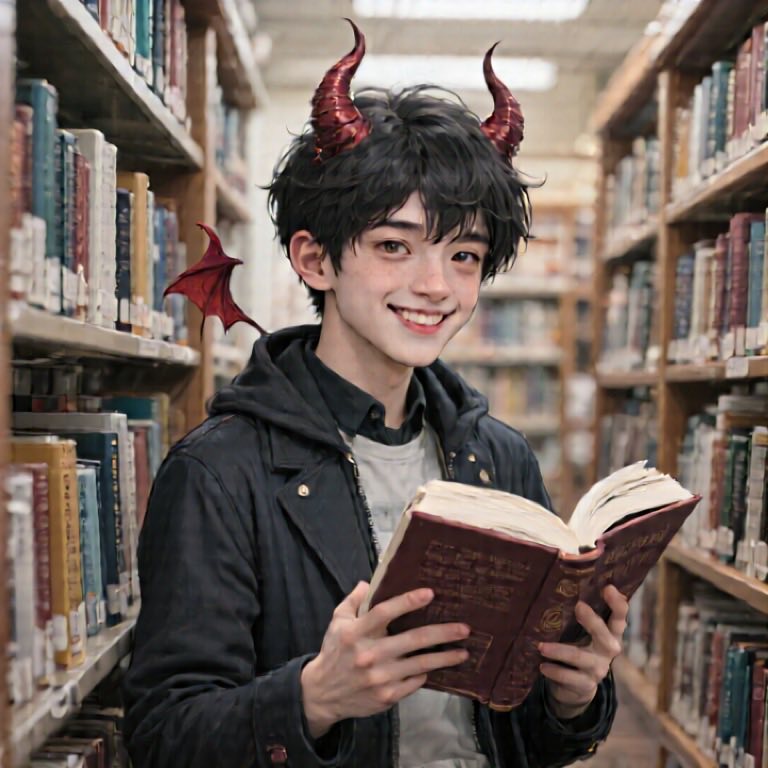}{A demon boy smiling while reading a book in a library.}
\qwenSampleRow{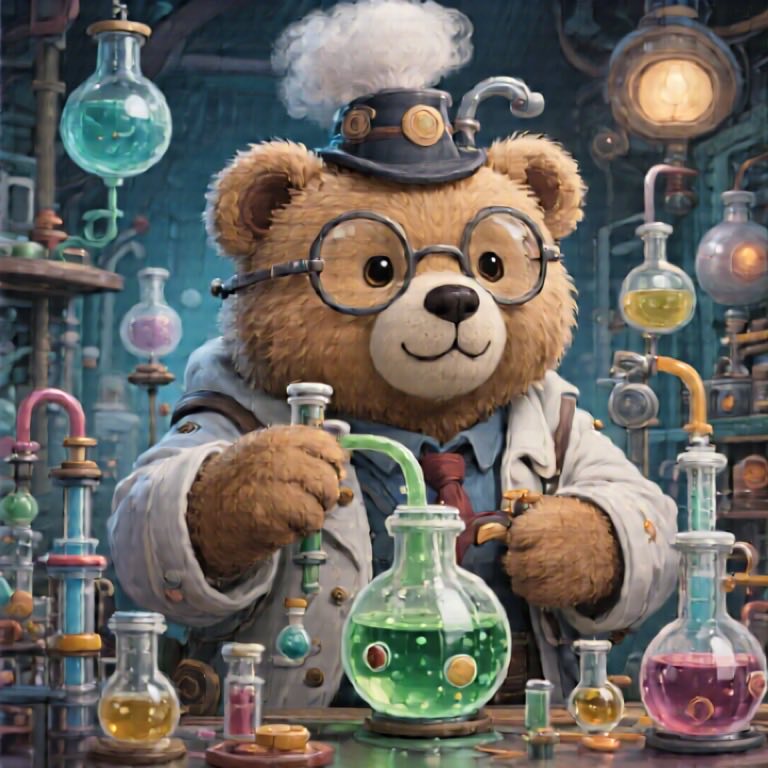}{A teddy bear mad scientist mixing chemicals depicted in oil painting style as a fantasy concept art piece.}
\qwenSampleRow{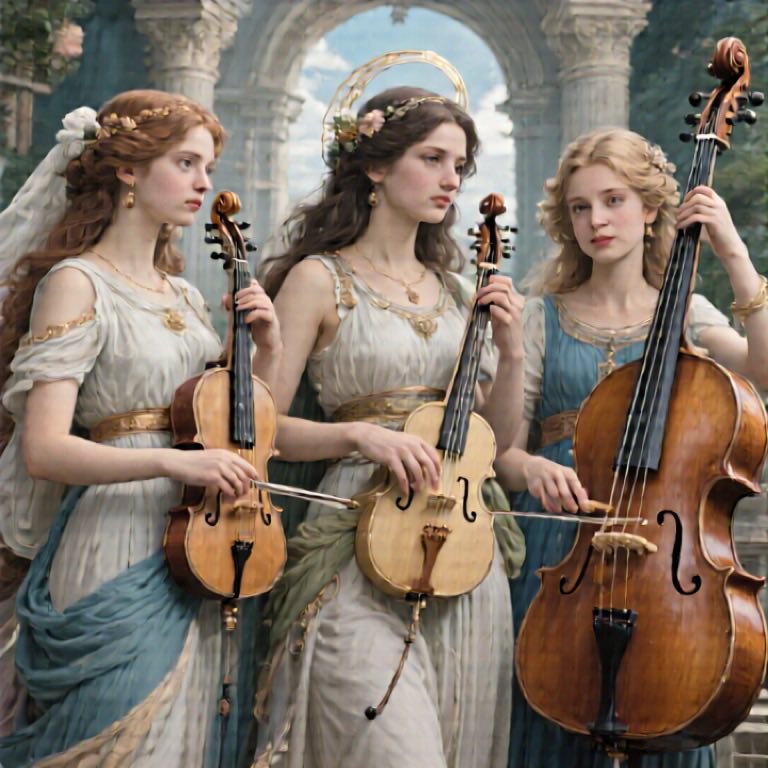}{The image depicts three female figures, known as the muses, playing musical instruments.}
\qwenSampleRow{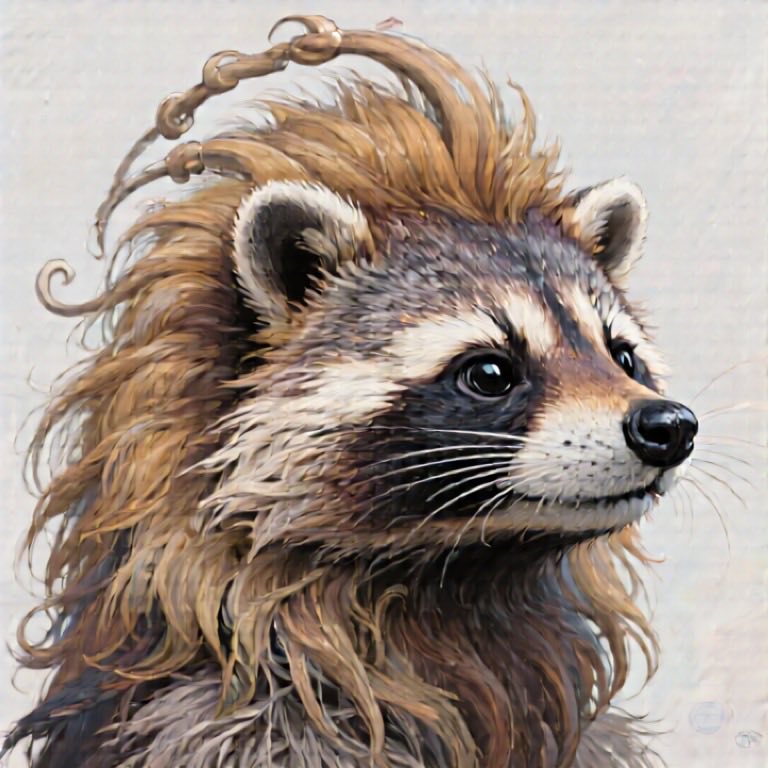}{A colorful, detailed painting of a raccoon with a long, flowing mane reminiscent of a lion's, styled in a mohawk.}

\subsection{Flux (Step 81)}

\fluxSampleRow{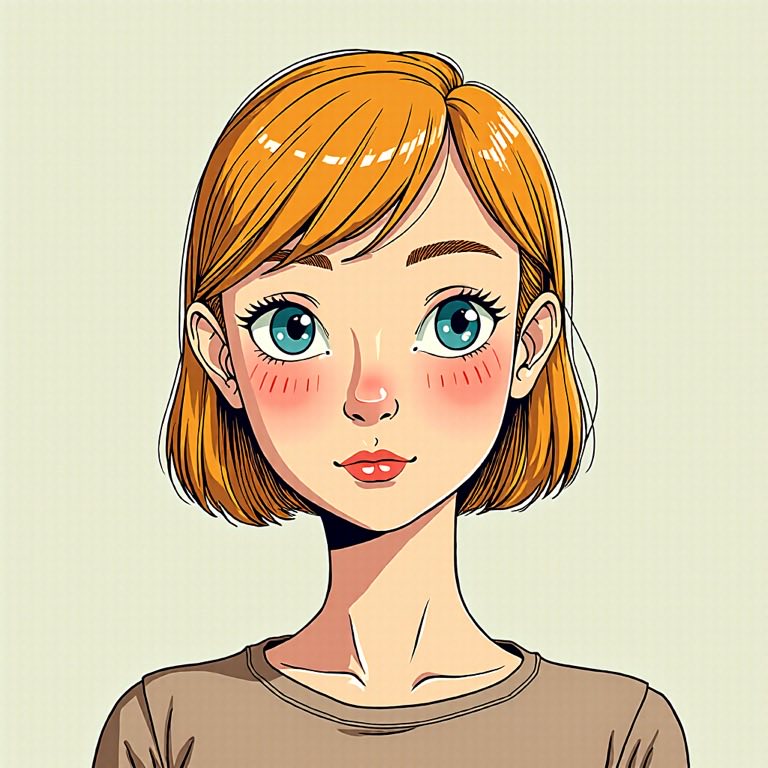}{A minimalist portrait of Chloe Grace by Jean Giraud in a comic style.}
\fluxSampleRow{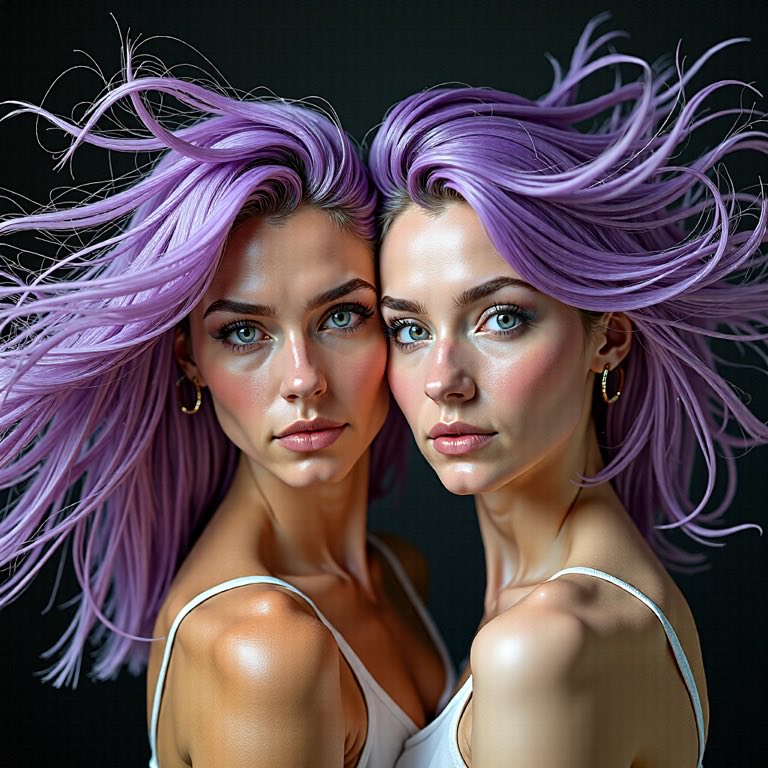}{A portrait of two women with purple hair flying in different directions against a dark background.}
\fluxSampleRow{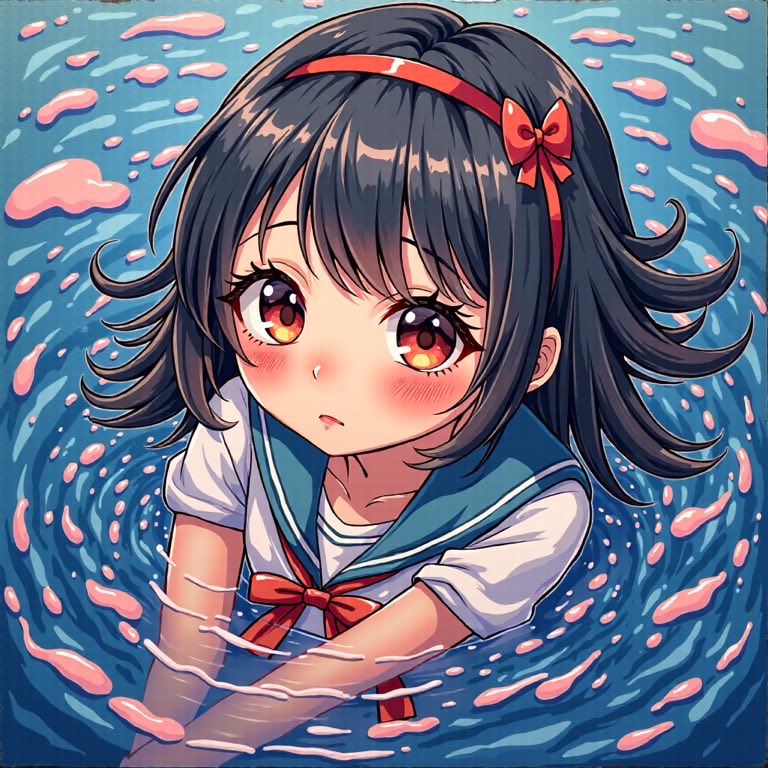}{A cute anime schoolgirl with a sad face submerged in dark pink and blue water, portrayed in an oil painting style.}
\fluxSampleRow{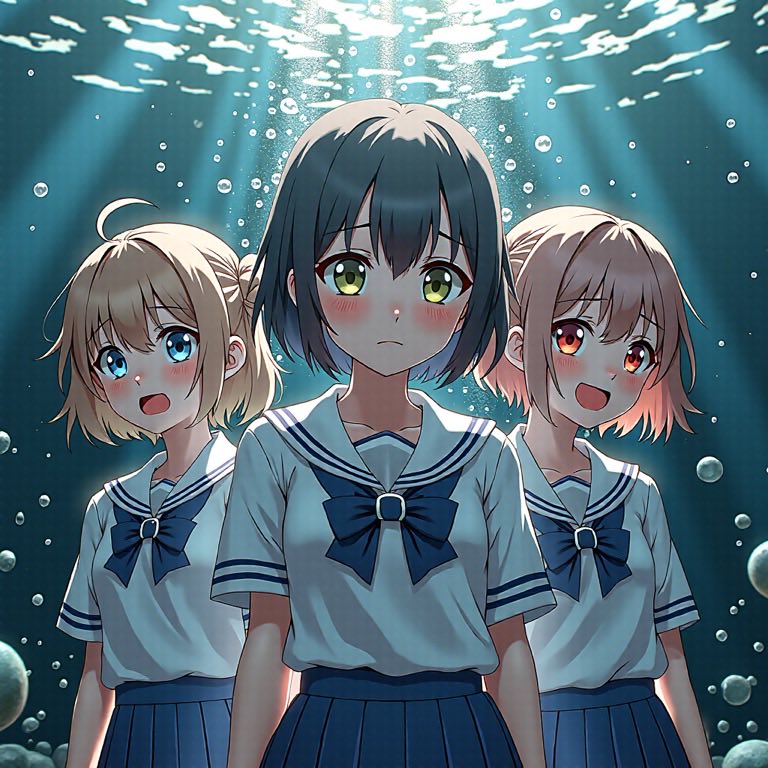}{A 3D rendering of anime schoolgirls with a sad expression underwater, surrounded by dramatic lighting.}
\fluxSampleRow{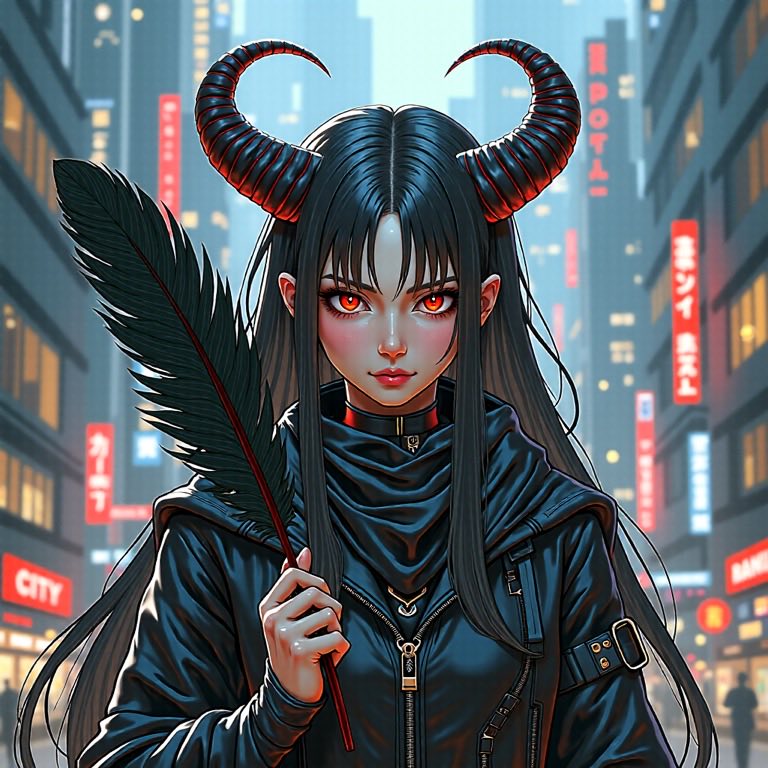}{A cyber girl with demon horns holds a black feather in front of a cybercity with a gloomy expression.}

\fluxSampleRow{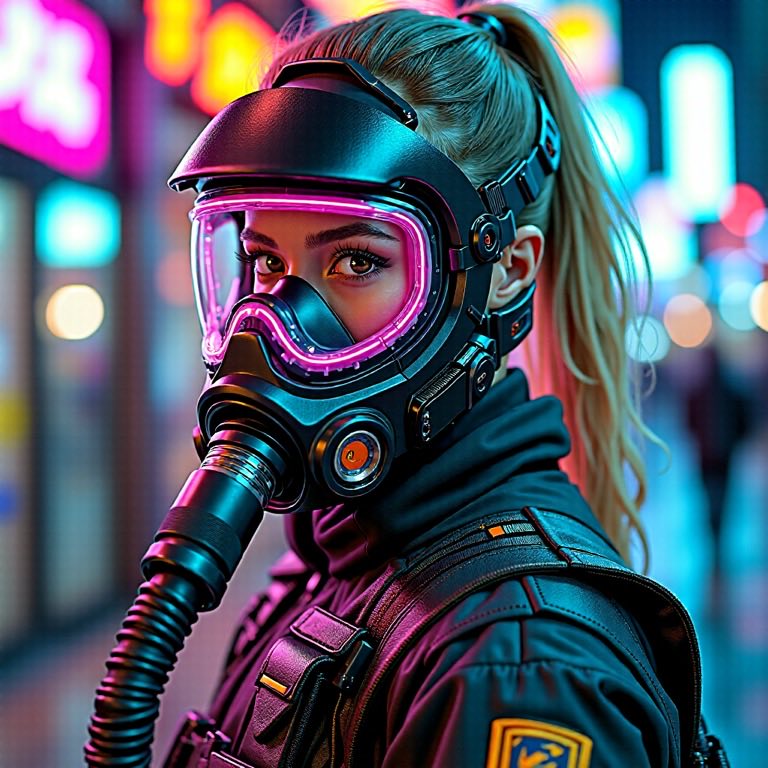}{A key visual of a young female swat officer with a neon futuristic gas mask in a cyberpunk setting.}
\fluxSampleRow{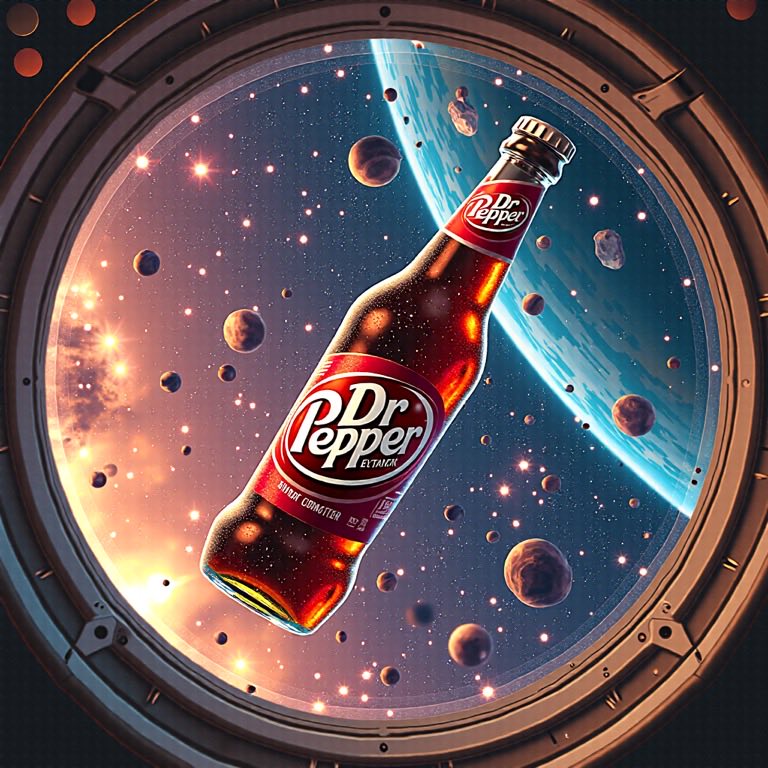}{Dr. Pepper floating in space, viewed through the window of a spaceship.}
\fluxSampleRow{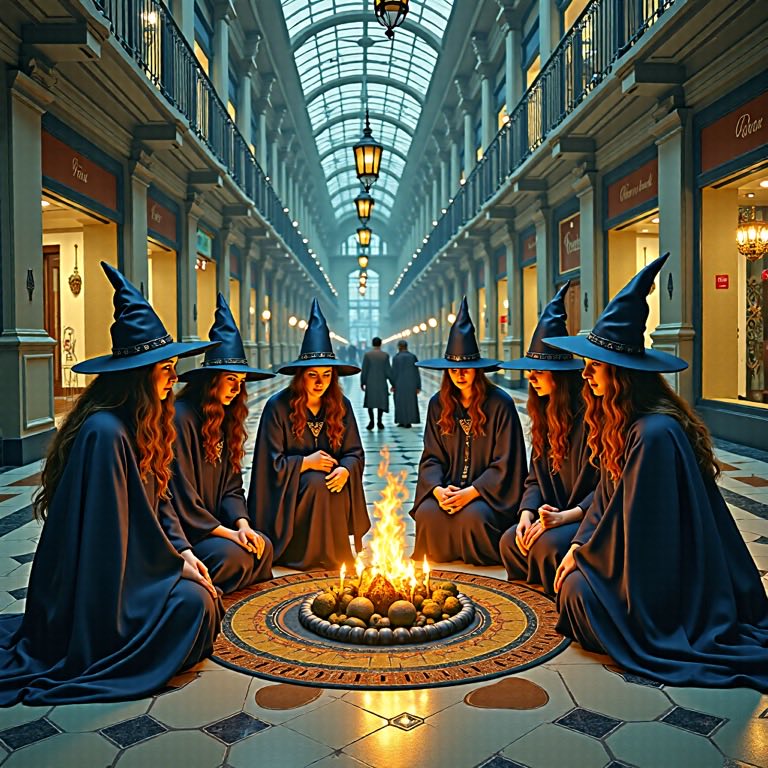}{Witches performing a ritual in a dark mall.}
\fluxSampleRow{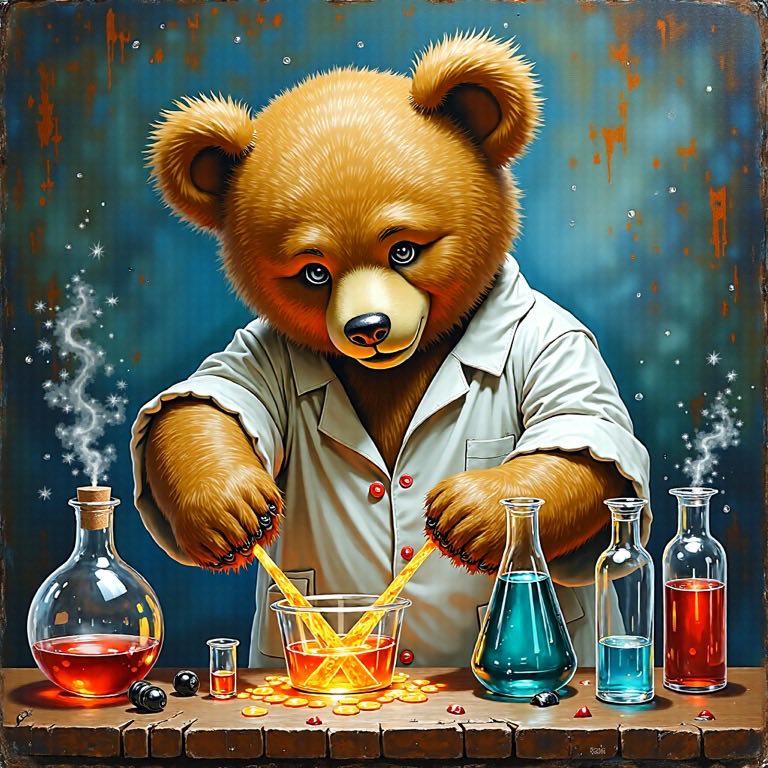}{A teddy bear mad scientist mixing chemicals depicted in oil painting style as a fantasy concept art piece.}
\fluxSampleRow{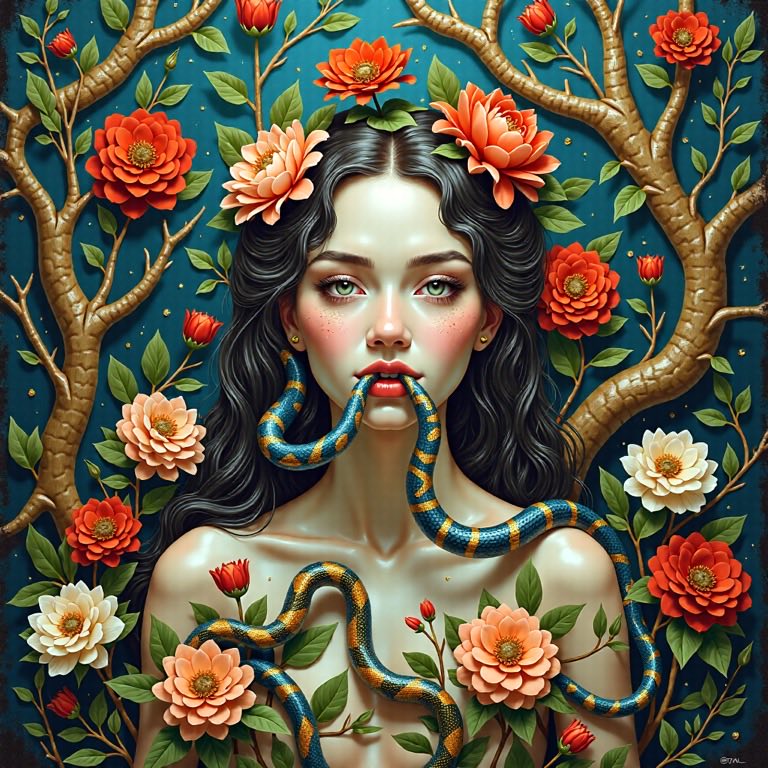}{Image of a woman with snakes in her mouth, surrounded by flowers and a twisted branch background, with a dark and moody atmosphere.}

\fluxSampleRow{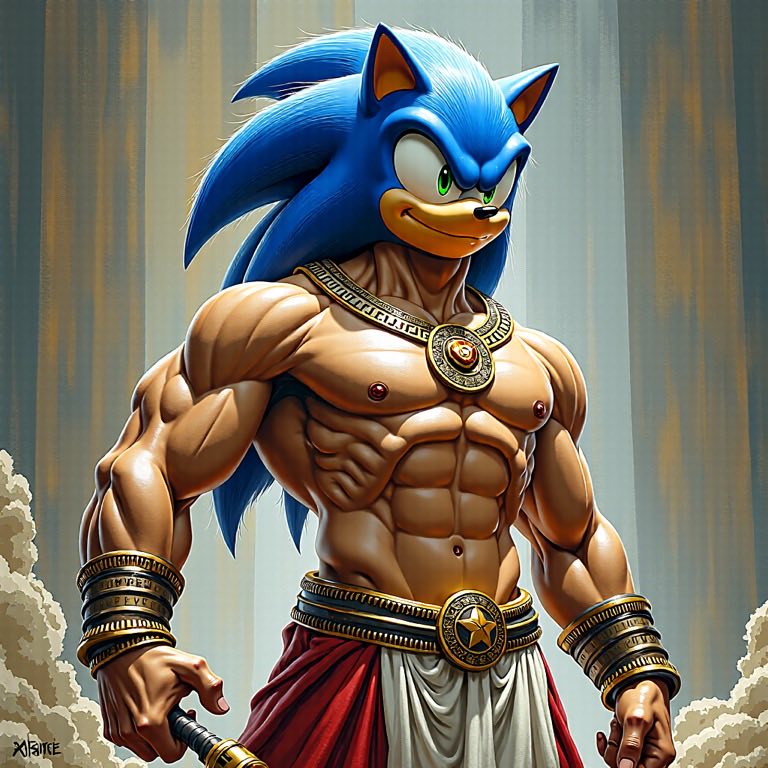}{Sonic the Hedgehog depicted as a muscular Greek god in a highly detailed digital painting.}
\fluxSampleRow{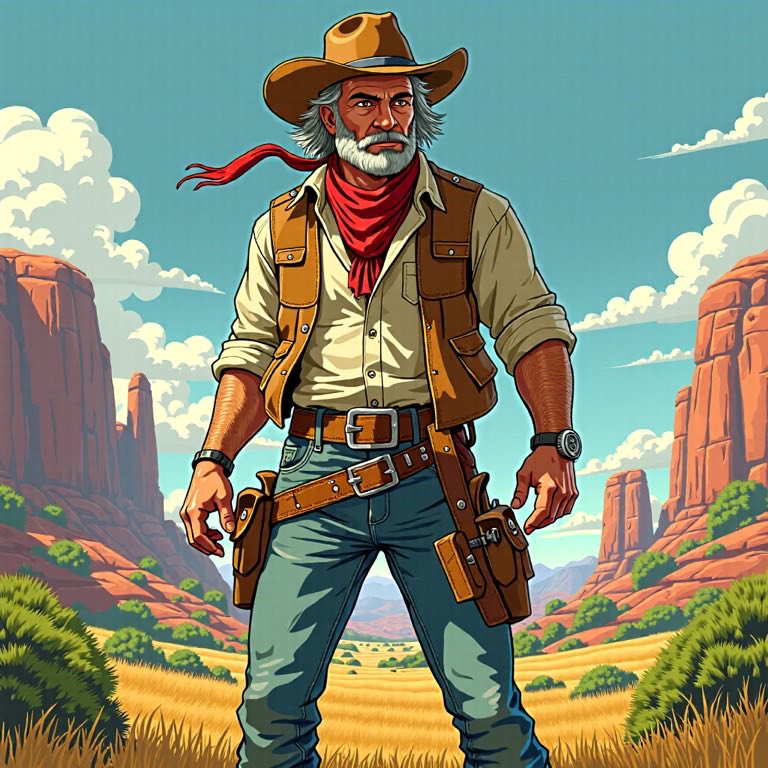}{American cowboy with a scruffy appearance in a retrofuturistic style, inspired by the animations of Studio Ghibli.}
\fluxSampleRow{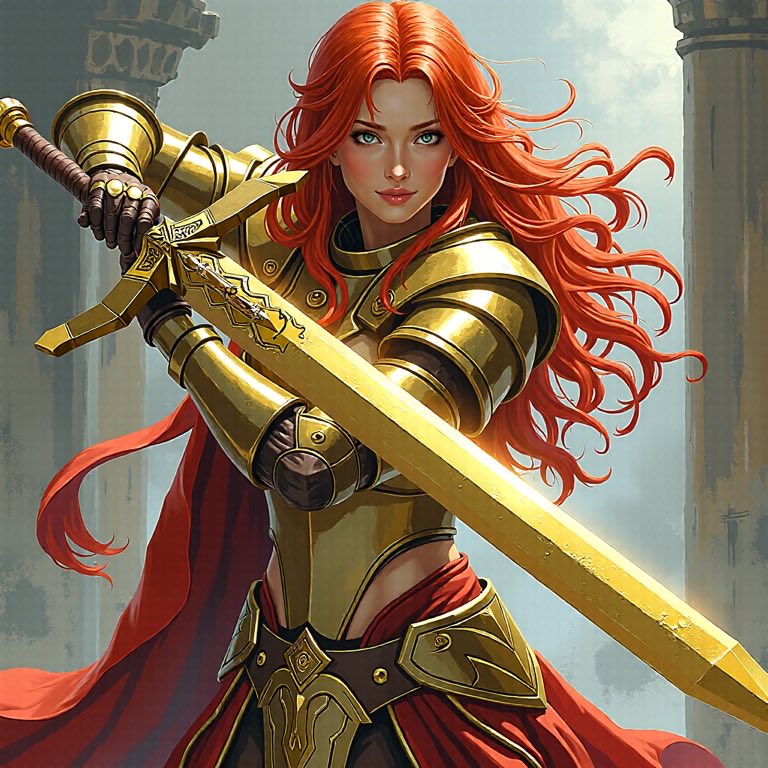}{A red-haired female knight with a golden prosthetic arm wields a long golden blade.}
\fluxSampleRow{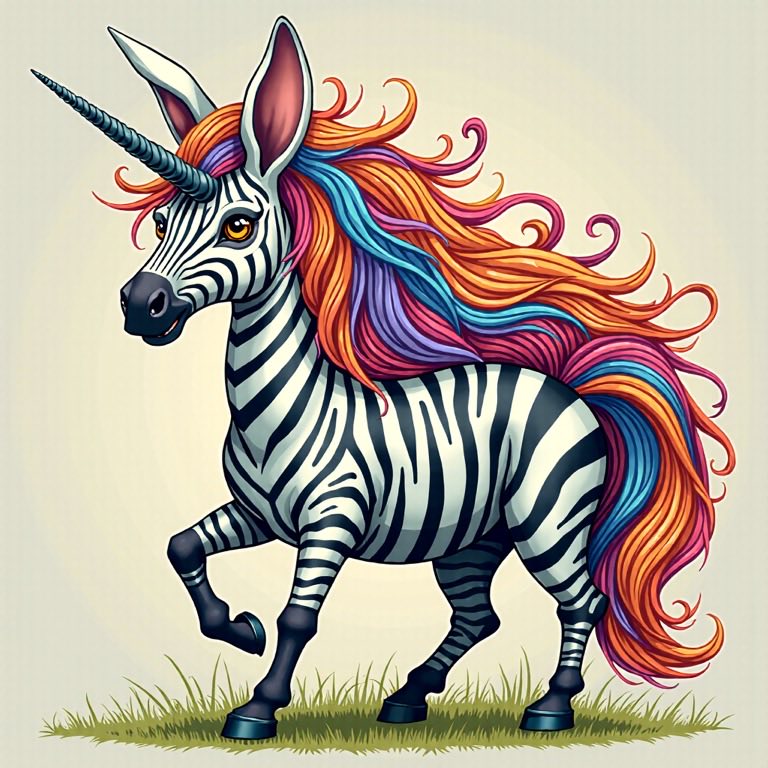}{A hybrid creature concept painting of a zebra-striped unicorn with bunny ears and a colorful mane.}
\fluxSampleRow{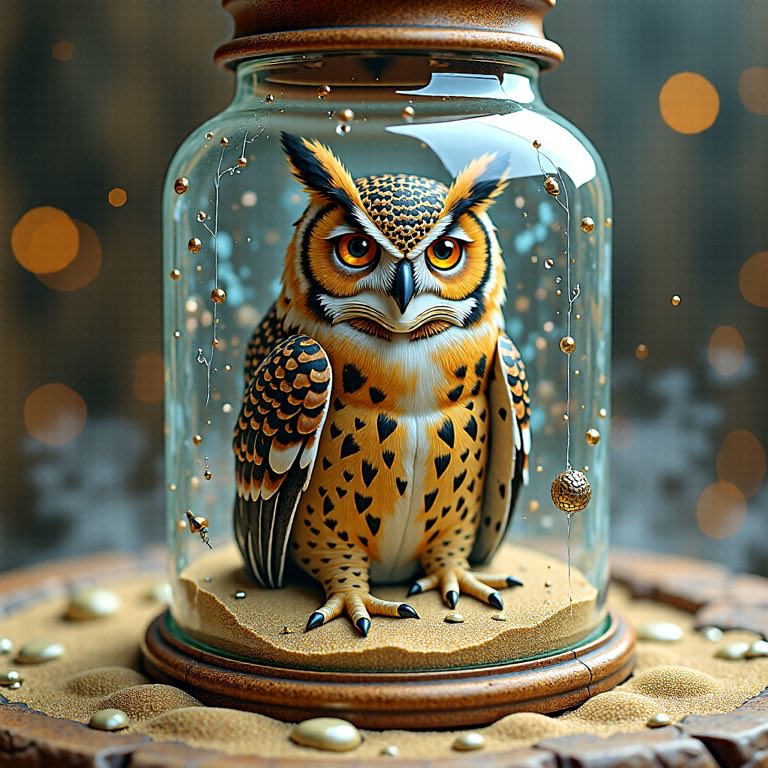}{A steampunk pocketwatch owl is trapped inside a glass jar buried in sand, surrounded by an hourglass and swirling mist.}

\fluxSampleRow{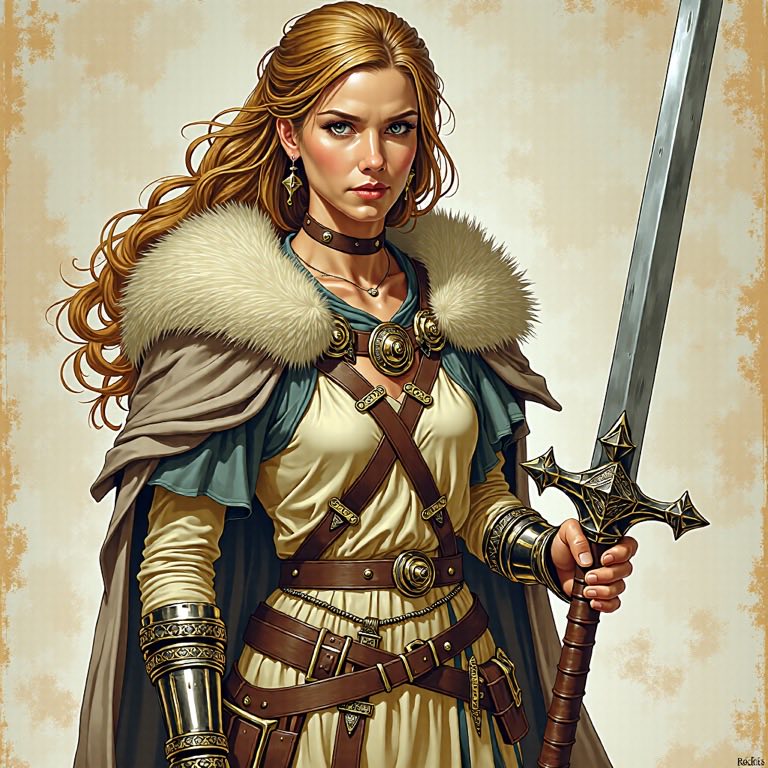}{A female human barbarian depicted in a traditional Dungeons and Dragons illustration.}
\fluxSampleRow{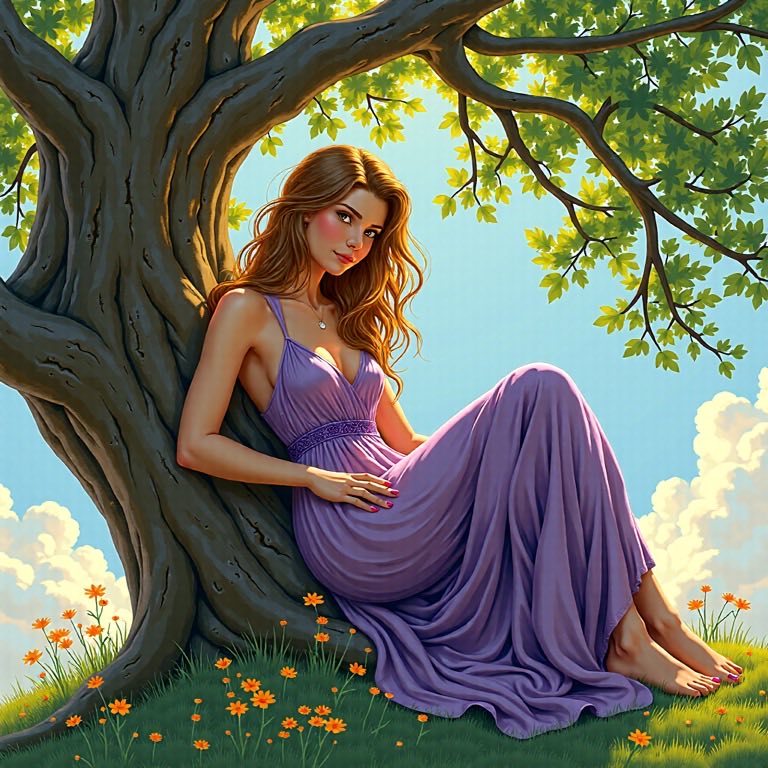}{A lady in a purple dress sitting in a tree --- concept art.}
\fluxSampleRow{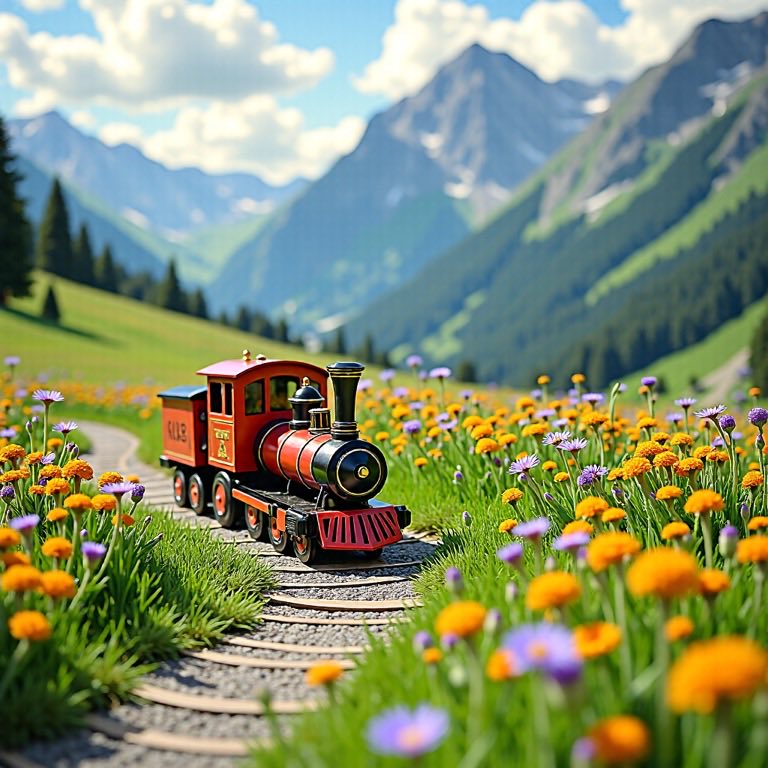}{A colorful tin toy robot runs a steam engine on a path near a beautiful flower meadow in the Swiss Alps with a mountain panorama in the background.}
\fluxSampleRow{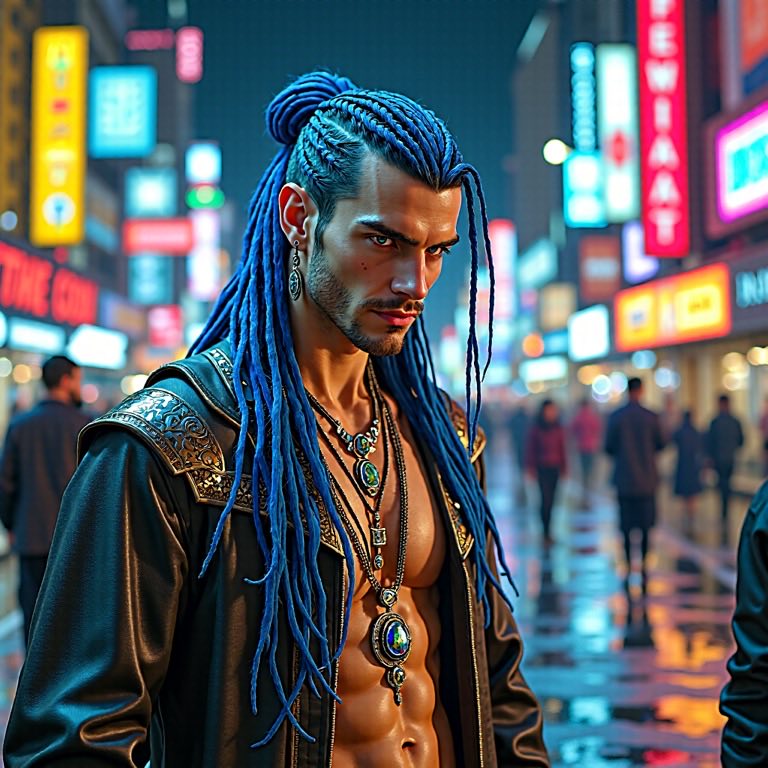}{Male vampire of clan Banu Haqim with blue braided hair stands in a modern city at night surrounded by neon signs, jewelry, and tattoos.}
\fluxSampleRow{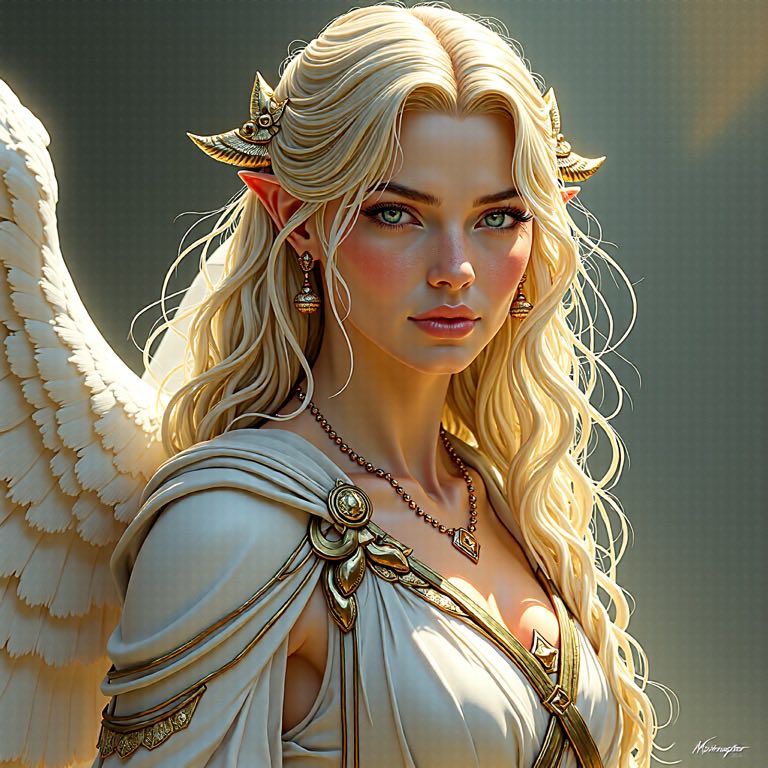}{The image is a digital art depiction of a female angel warrior with detailed features.}

\FloatBarrier
\section{Flow-GRPO-Fast Qualitative Samples}
\label{sec:appendix-flow-fast-samples}

This section provides additional qualitative comparisons for QwenImage and Flux trained with the Flow-GRPO-Fast algorithm. Each row compares outputs generated from the same prompt and shared trajectory ID. QwenImage samples are generated at step~181 and compare the standard Flow-GRPO-Fast baseline against JAGG. Flux samples are taken from step~181 and additionally include the endpoint-only ablation, highlighting the qualitative effect of using the full JAGG aggregation rather than simply back-propagating through two endpoints.

\subsection{QwenImage on Flow-GRPO-Fast (Step 181)}

\flowFastQwenSampleRow{pair_01}{00327}{A painting of a woman by Zinaida Serebriakova wearing a T-shirt with the Supreme brand logo, a sleeveless white blouse, dark brown capris, and black loafers.}
\flowFastQwenSampleRow{pair_02}{00396}{A nightstand topped with a white land-line phone, remote control, a metallic lamp, and two pens next to a black hardcover book.}
\flowFastQwenSampleRow{pair_03}{00256}{Close-up of Cad Bane, with bad flash.}
\flowFastQwenSampleRow{pair_04}{00308}{A lemon with a McDonald's hat.}
\flowFastQwenSampleRow{pair_05}{00062}{A lemon wearing a suit and tie, full body portrait.}
\flowFastQwenSampleRow{pair_06}{00076}{Elvis Presley performing in a jumpsuit, artwork by Alessandro Pautasso.}
\flowFastQwenSampleRow{pair_07}{00224}{Underwater concept art of marine life in Sea of Thieves featuring a wild boar.}
\flowFastQwenSampleRow{pair_08}{00057}{A still frame from the anime film Akira.}
\flowFastQwenSampleRow{pair_09}{00266}{A 3D render of a volcanic icon on a rocky background, in isometric perspective and darkly lit.}
\flowFastQwenSampleRow{pair_10}{00240}{A mother bear and her cub crossing a two lane road.}

\subsection{Flux on Flow-GRPO-Fast (Step 181)}

\flowFastFluxSampleRow{pair_01}{00176}{Witches performing a ritual in a dark mall.}
\flowFastFluxSampleRow{pair_02}{00336}{A high shot of many people standing in an airport.}
\flowFastFluxSampleRow{pair_03}{00188}{A photograph of a giant diamond skull in the ocean, featuring vibrant colors and detailed textures.}
\flowFastFluxSampleRow{pair_04}{00081}{Two men sitting in a green living room talking to a girl seen in the mirror.}
\flowFastFluxSampleRow{pair_05}{00314}{A vehicle with many people even on top.}
\flowFastFluxSampleRow{pair_06}{00322}{A tall giraffe in a zoo eating branches.}
\flowFastFluxSampleRow{pair_07}{00330}{A portrait of a dinner dish of a protein and greens.}
\flowFastFluxSampleRow{pair_08}{00395}{Portrait of anime girl in mechanic armor in night Tokyo.}
\flowFastFluxSampleRow{pair_09}{00184}{A portrait of Nyan Cat, styled after Annie Leibovitz's dramatic photography.}
\flowFastFluxSampleRow{pair_10}{00190}{A ginger haired mouse mechanic in blue overalls in a cyberpunk scene with neon slums in the background.}

\end{document}